\documentclass[review]{elsarticle}

\usepackage{graphics} 
\usepackage{epsfig} 

\usepackage{physics}
\usepackage{amsmath}
\usepackage{amsfonts}
\usepackage{amssymb}
\usepackage{amsthm}
\usepackage{arydshln}
\usepackage{graphicx}
\usepackage{color}
\usepackage{algorithm}
\usepackage{algpseudocode}
\usepackage{varwidth}
\usepackage{url}
\usepackage{multirow}
\usepackage{natbib}

\allowdisplaybreaks

\ifodd 1

\newcommand{\com}[1]{{\color{red}{Comment: #1}}}
\else
\newcommand{\com}[1]{}
\fi

\algdef{SE}[DOWHILE]{Do}{doWhile}{\algorithmicdo}[1]{\algorithmicwhile\ #1}%

\newtheorem{theorem}{Theorem}
\newtheorem{lemma}{Lemma}

\newtheorem{corollary}{Corollary}
\newtheorem{assumption}{Assumption}

\begin{document}

\begin{frontmatter}

\title{EventGraD: Event-Triggered Communication in Parallel Machine Learning}

\author{Soumyadip Ghosh\corref{cor1}}\ead{sghosh2@nd.edu}
\cortext[cor1]{Corresponding author}
\author{Bernardo Aquino}\ead{bcruz2@nd.edu}
\author{Vijay Gupta}\ead{vgupta2@nd.edu}
\address{Department of Electrical Engineering, University of Notre Dame, USA }

\begin{abstract}

Communication in parallel systems imposes significant overhead which often turns out to be a bottleneck in parallel machine learning. To relieve some of this overhead, in this paper, we present EventGraD - an algorithm with event-triggered communication for stochastic gradient descent in parallel machine learning. The main idea of this algorithm is to modify the requirement of communication at every iteration in standard implementations of stochastic gradient descent in parallel machine learning to communicating only when necessary at certain iterations. We provide theoretical analysis of convergence of our proposed algorithm. We also implement the proposed algorithm for data-parallel training of a popular residual neural network used for training the CIFAR-10 dataset and show that EventGraD can reduce the communication load by up to 60\% while retaining the same level of accuracy. In addition, EventGraD can be combined with other approaches such as Top-K sparsification to decrease communication further while maintaining accuracy.

\end{abstract}


\end{frontmatter}

\section{Introduction}
\label{sec:intro}

Artificial intelligence in general, and machine learning in particular, is revolutionizing many aspects of our life~\cite{gorriz2020artificial}. Machine learning (ML) algorithms in various applications have achieved significant benefits through training of a large number of parameters using huge data sets. 
Focus has now shifted to ensure that these algorithms can be executed for complex problems in a reasonable amount of time. Initial speed-ups in ML algorithms were due to better algorithm design (e.g. using mini-batches) or hardware (e.g. introduction of graphics processing units(GPUs)). However, to stay relevant, machine learning must continue to scale up in the size and complexity of the application problem. The challenge is both in the large number of parameters that need to be trained and the consequent large amount of data that needs to be processed.

An obvious answer is to go from one processing element -- that may have neither the memory nor the computational capability needed for machine learning implementations to solve complex problems -- to multiple processing elements (sometimes referred to as parallel or distributed implementations)~\cite{upadhyaya2013parallel,ben2019demystifying}.  For instance, there has been a lot of recent interest in machine learning using artificial neural networks on large-scale clusters such as supercomputers~\cite{young2017evolving,yin2019strategies}. Both data-parallel (in which the dataset is divided into multiple processors, with each processor having a copy of the entire neural network model) and model-parallel (in which the neural network model is divided among multiple processors, with each processor having access to the entire dataset) architectures have been considered~\cite{bekkerman2011scaling}. For some applications such as federated learning which involves edge devices such as smartphones or smart speakers, distributed training is often the only choice due to data privacy concerns~\cite{boulemtafes2020review}.  

One of the biggest challenges of training in any parallel or distributed environment is the overhead associated with communication between different processors or devices. In high performance computing clusters, communication of messages over networks often takes a lot of time, consumes significant power and can lead to network congestion~\cite{bergman2008,lucas2014doe,jana2014power}. Specifically for parallel machine learning, during training, the processors need to exchange the weights and biases with each other before moving to the next training iteration. For example, in a data parallel architecture, the weights and biases among the different processors are averaged with each other (either directly or through a central parameter server) before executing the next training iteration. Such an exchange usually happens by message passing at the end of every iteration. As the number of processing elements increases, the issue of such communication being a major bottleneck in these implementations is known widely~\cite{zhang2013communication,alistarh2016qsgd,lin2018deep,seide20141-bit}.

Consequently there has been a lot of research aimed at reducing communication in parallel machine learning~\cite{zhang2013communication,alistarh2016qsgd,lin2018deep,seide20141-bit,liu2021consensus}. This rich stream of work has suggested various ways of reducing the size or number of messages as means of alleviating the communication overhead. In this paper, we propose a novel algorithm to reduce communication in parallel machine learning involving artificial neural networks. Specifically, we utilize the idea of event-triggered communication from control theory to design a class of communication-avoiding machine learning algorithms. In this class of algorithms, communication among the processing elements occurs intermittently and only on an as-needed basis. This leads to a significant reduction in the number of messages communicated among the processing elements. Note that algorithms to reduce communication have been proposed in other applications of parallel computing as well, such as parallel numerical simulation of partial differential equations~\cite{chronopoulos1989,hoemmen2010,ghosh2018event}. 

The core idea of our algorithm is to exchange the neural network parameters (the weights and biases) only when a certain criterion related to the utility of the information being communicated is satisfied, i.e., in an event-triggered fashion. We present both theoretical analysis and experimental demonstration of this algorithm. Experimentally, we show that the algorithm can yield the same accuracy as standard implementations with $60\%$ lesser number of messages transmitted among processors using a popular residual neural network on the CIFAR-10 dataset. Our implementation is open-source and available at~\cite{code_url}. A reduction in the number of messages implies a reduction in both the time and energy overhead of communication and can prevent congestion in the network. Theoretically, we show that our algorithm has a bound on the convergence rate (in terms of number of iterations) of the order $\mathcal{O}\left(\frac{1}{\sqrt{Kn}} + \frac{G(K-1)}{\sqrt{K}} + \frac{G_{1/2}^2(K-1)}{\sqrt{K}} \right)$ where $K$ is number of iterations, $n$ is number of processors, and $G(K)$ and $G_{1/2}(K)$ are terms related to a bound on the threshold of event-triggered communication. In particular, if the bound on the threshold is chosen to be a sequence that decreases geometrically as a function of the iteration number, the bound on the convergence rate becomes of the order $\mathcal{O}\left(\frac{1}{\sqrt{Kn}} + \frac{1}{\sqrt{K}}\right)$ which is similar to the asymptotic convergence rate of parallel stochastic gradient descent in general. An earlier version of this algorithm without theoretical results was experimentally demonstrated as a proof of concept on the smaller MNIST dataset in~\cite{ghosh2020eventgrad}. In contrast, this paper contains a comprehensive theoretical treatment of the algorithm with additional experiments on the CIFAR-10 dataset. In our previous work~\cite{ghosh2018event}, we have considered event-triggered communication for a different domain of parallel numerical partial differential equation solvers and highlighted the implementation challenges similar to this paper. However we considered a fixed threshold without any mathematical treatment in that work unlike the adaptive threshold along with theoretical convergence results studied in this paper.

While we focus on data-parallel stochastic gradient descent in parallel machine learning for theoretical analysis and experimental verification of the algorithm in this paper, the idea of event-triggered communication can be applied to model-parallel and hybrid configurations and can be extended to other training algorithms as well such as Adam, RMSProp, etc. Similarly, event-triggered communication can also be used in federated learning where communication can have a more severe overhead due to the geographical separation between the devices involved in training such as smartphones. 

The paper is organised as follows. Section~\ref{sec:related} surveys related work and Section~\ref{sec:probform} introduces the necessary background. The proposed algorithm is introduced in Section~\ref{sec:idea} with theoretical analysis in Section~\ref{sec:analysis} and implementation details in Section~\ref{sec:impl}. Section~\ref{sec:results} contains the experimental results followed by conclusion in Section~\ref{sec:conc}. For notational convenience, we denote the abbreviation PE to be a processing element that signifies one core of a processor.

\section{Related Work}
\label{sec:related}

In this section, we review some communication-efficient strategies of distributed training of neural networks from literature. We primarily focus on parallel stochastic gradient descent~\cite{zinkevich2010parallelized, bottou2018optimization}. We also review some works on event-triggered communication and then highlight our specific contributions on using event-triggered communication for parallel training of neural networks.

\textbf{Parameter Server - } A popular approach for parallelization of stochastic gradient descent is the centralized parameter server approach where multiple workers compute the gradients in their assigned sub-dataset and send them to a central parameter server. The parameter server updates the neural network model parameters using the individual gradients and sends the updated parameters back to the workers, who then move on to the next iteration. The original approach results in a synchronized algorithm. This requirement of synchronization was relaxed by the Hogwild algorithm~\cite{recht2011hogwild} where the worker PEs can send gradients to the parameter server asynchronously without any lock step. Elastic Averaging SGD proposed by~\cite{zhang2015deep} reduces communication by introducing the notion of an elastic period of communication between the workers and the parameter server. Other approaches have also been proposed in the literature~\cite{lian2015asynchronous}.
There have been studies to reduce communication with the parameter server in the context of federated learning as well~\cite{sattler2019robust,chen2019communication,xu2020ternary}. However, the parameter server approach often suffers from poor scalability due to the dependence on a central node, which can become a bottleneck.

\textbf{AllReduce - } Another popular approach for parallelization of stochastic gradient descent does not consider a centralized parameter server. Rather, every PE maintains a copy of the model and the PEs average the parameters of the model by communicating in an all-to-all fashion among themselves using a reduction mechanism commonly known as AllReduce~\cite{gropp1999using}. Since such all-to-all communication incurs a lot of overhead, a lot of research has focused on reducing this overhead by using optimized variants. The authors in~\cite{seide20141-bit} have proposed one-bit quantization where each gradient update is quantized to 1-bit, resulting in a reduction of the data volume that needs to be communicated. Threshold quantization was developed in~\cite{strom2015scalable} where only those gradient updates that are greater than a specified threshold are transmitted after encoding them with a fixed value. A hybrid approach combining both 1-bit and threshold quantization was given in the adaptive quantization proposed in~\cite{dryden2016communication}. Deep Gradient Compression in~\cite{lin2018deep} compresses the size of gradients and accumulates the quantization error and momentum to maintain accuracy. Several approaches have been proposed to minimize communication by reducing the precision of gradients, e.g., using half precision (16-bit) for training~\cite{gupta2015deep} and mixed precision~\cite{micikevicius2018mixed}. Sparsified methods that communicate only the top-k most significant values have been proposed by~\cite{alistarh2018convergence,renggli2019sparcml}.
Combining the two methods of quantization and sparsification is presented in~\cite{basu2019qsparse}.

\textbf{Reduction with neighbors - } Instead of averaging the parameters using all-to-all communication among all the PEs via AllReduce, another approach has been proposed where the averaging is done only with the neighboring PEs in the topology in which the PEs are connected~\cite{yuan2016convergence}. This approach uses ideas from consensus algorithms which is now widely studied in many different communities~\cite{olfati2007consensus}. We would like to point out that there is confusion in literature about the term ``decentralized" - some works call the AllReduce approach involving averaging among all PEs as decentralized because of the absence of a central parameter server~\cite{li2018pipe}, while others call the approach that involves averaging with just the neighboring PEs as decentralized because it does not require any central operation on all the PEs in the topology~\cite{yuan2016convergence,lian2017can}. We adopt the latter usage and call the algorithms that require averaging with just the neighboring PEs as decentralized. While it might seem that such a scheme will converge slower than a centralized approach due to the delayed dissemination of information to all the nodes, the authors in~\cite{lian2017can} showed that the convergence rate is similar in order to the centralized approach after the same number of iterations provided the number of iterations is sufficiently large. While~\cite{lian2017can} considers that the neighbors of a particular PE remain fixed across iterations, there are interesting gossip algorithms~\cite{blot2016gossip,jin2016scale,daily2018gossipgrad} that choose neighbors randomly and exchange information. More recently, the authors in~\cite{liu2021consensus} propose an error-compensated communication compression mechanism in this context using bit-clipping that reduce communication costs.

\textbf{Event-Triggered Communication} - Event-Triggered Communication has been proposed as a mechanism for reducing communication in networked control systems~\cite{lemmon2010event,dimarogonas2012distributed}. Methods employing event-triggered communication in consensus algorithms have been variously proposed, e.g.,~\cite{wang2019distributed,chen2017event,nowzari2019event}. Closely related to consensus is the problem of distributed optimization where there have been various event-triggered approaches proposed~\cite{zhao2018distributed,lu2017distributed, richert2016distributed}. Of particular relevance is~\cite{zhang2016adaptive} which suggests an event-triggered communication scheme with an adaptive threshold of communication that is dependent on the state of the last trigger instant in a continuous time control system.

\textbf{Our Contribution - }The decentralized parallel stochastic gradient descent in~\cite{lian2017can} still considers that communication of parameters with neighbor PEs happens at every iteration. Since the values of the parameters may not change significantly in every iteration, communication at every iteration may not be necessary. Thus the main idea behind the algorithms presented in this paper is to communicate these parameters in events only when their values change by a certain threshold. We consider the scenario of data-parallel training of a neural network in a high performance computing (HPC) cluster where there are fixed neighboring PEs for every PE and show that communicating in events with neighboring PEs reduces the number of messages passed in the network. Decreasing the message count decreases the overall data to be communicated and as pointed out in literature~\cite{zhang2013communication,alistarh2016qsgd,lin2018deep,seide20141-bit,liu2021consensus}, reducing the data to be communicated reduces the overhead associated with communication. More concretely, the contributions of our work are:
\begin{itemize}
\item We propose an event-triggered communication algorithm where the neural network parameters, i.e., the weights and biases, are communicated only when their norm changes by some threshold. The threshold is chosen in an adaptive manner based on the rate of change of the parameters.  
\item We derive an expression for a bound on the convergence rate of event-triggered communication based on a generic bound on the adaptive threshold.
\item We provide an open-source high performance computing (HPC) implementation of our algorithm using PyTorch and Message Passing Interface (MPI) in C++. We also highlight implementation challenges of this algorithm, particularly the need of advanced features such as one-sided communication, also called remote memory access and the requirement of the newer PyTorch C++ frontend over its traditional Python frontend. We believe that it is not possible to implement event-triggered communication without remote memory access in any computer network as elaborated later in Section~\ref{sec:impl}. Our implementation is open-source and available at~\cite{code_url}.  
\end{itemize}

The paper closest to ours seems to be~\cite{george2019distributed} where the authors considered a federated learning scenario and
proposed an event-triggered communication scheme for the model parameters based on thresholds that are dependent on the learning rate and showed reduction in communication for distributed training. 
As compared to that work, we consider an adaptive threshold rather than selecting the same threshold across all parameters. In particular, the threshold is adaptive to the local slope of a parameter and thus it can adjust according to the parameter's evolution which will depend on factors such as the type of the parameter, the neural network model and the dataset. Hence the adaptive threshold makes our algorithm robust to different neural network models and different datasets. Our theoretical results are based on a generic bound on the threshold unlike~\cite{george2019distributed} which provides a bound considering a certain form of threshold dependent on the learning rate. Further, we highlight the implementation challenges of event-triggered communication in an HPC environment which is different than the federated learning setting considered in~\cite{george2019distributed} that usually involves wireless communication.
 
\section{Problem Formulation}
\label{sec:probform}

This section lays the mathematical preliminaries for the main algorithm introduced in the next section. We consider a decentralized communication graph $(V, W)$ where $V$ denotes the set of $n$ PEs and $W \in \mathbb{R}^{n \times n}$ is the symmetric doubly stochastic adjacency matrix. $W_{ii}$ corresponds to the weight of the state of the $i$-th PE while $W_{ij}$ corresponds to the weight of the state of the $j$-th PE on the state of the $i$-th PE. We assume that $W$ represents a ring topology that remains fixed throughout iterations. Now the objective of data-parallel training of any neural network can be expressed as 
\begin{align}
\min_{x\in \mathbb{R}^n}{f(x)=\frac{1}{n}\sum_{i=1}^{n}\underbrace{{\mathbb{E}_{\xi \sim \mathbb{D}}F_i(x;\xi)}}_{:=f_i(x)}}
\end{align}
where $\mathbb{D}$ is the sampling distribution, considered to be the same in every PE. Further, the neural network in every PE is considered to have $N$ parameters. These parameters are the weights and biases in the neural network model.

For mathematical formulation, let us define the concatenation of all local parameters $X_k$, random samples $\xi_k$, stochastic gradients $\partial F(X_k; \xi_k)$ and expected gradients $\partial f(X_k)$ as : 
\begin{align*}
&X_k := \left[ x_{k,1} \quad \cdots \quad  x_{k,n} \right] \in \mathbb{R}^{N \times n}, \quad \xi_k:=\left[ \xi_{k,1} \quad \cdots \quad  \xi_{k,n} \right]^\top \in \mathbb{R}^{n}, \\
&\partial F\left( X_k;\xi_k \right) := \left[ \nabla F_1\left( x_{k,1};\xi_{k,1} \right) \quad \nabla F_2\left( x_{k,2};\xi_{k,2} \right) \quad \cdots \quad  \nabla F_n\left( x_{k,n};\xi_{k,n} \right) \right] \in \mathbb{R}^{N \times n}, \\
&\partial f\left( X_k \right) := \left[ \nabla f_1\left( x_{k,1} \right) \quad \nabla f_2\left( x_{k,2} \right) \quad \cdots \quad  \nabla f_n\left( x_{k,n} \right) \right] \in \mathbb{R}^{N \times n}.
\end{align*}
The algorithm is said to converge to a $e$-approximate solution if
\begin{align*}
K^{-1}\left( \sum_{k=0}^{K-1}{\mathbb{E}\left \Vert \nabla f \left( \frac{X_k\textbf{1}_n}{n} \right) \right \Vert^2} \right)\leq e.
\end{align*}
Now the training algorithm for the decentralized stochastic gradient descent mentioned in~\cite{lian2017can}
can be expressed as:
\begin{align}
\label{eqn:dist_sgd_neigh}
X_{k+1} = X_{k}W-\gamma \partial F\left( X_k;\xi_k \right),
\end{align}
where $\gamma$ is the step size or learning rate. From~\eqref{eqn:dist_sgd_neigh}, it is clear that values from neighbor PEs are needed to calculate the values in a particular PE. Thus the parameters of the neural network, i.e., the weights and biases, are communicated between the neighbor PEs after every iteration. For details on how to choose $W$ optimally, the reader is referred to~\cite{boyd2004fastest}. Usually the values of $W$ are taken to be $\frac{1}{\mathcal{N}_i + 1}$ where $\mathcal{N}_i$ are the number of neighbors of the $i$-th PE. This means that the parameters at the $i$-th PE are averaged with that of its neighbors after every iteration. For the ring topology that we assume, $\mathcal{N}_i = 2$ for all $i$. After training concludes, the models in all the PEs are usually averaged to produce one model which is then evaluated on the test dataset. This algorithm from~\cite{lian2017can}, named D-PSGD in that paper, is stated in pseudo code in Algorithm~\ref{alg:regular_ml}. Since communication between neighboring PEs happen \textit{regularly} after every iteration, we refer to this algorithm as the one with regular communication. We modify this algorithm to include event-triggered communication as proposed in the next section.

\begin{algorithm}
\renewcommand\thealgorithm{\Alph{algorithm}}
\caption{: Regular Communication in Data Parallel Machine Learning}
\label{alg:regular_ml}
\begin{algorithmic}
\For{$k = 0, 1, 2, \dots K-1$}
\State Randomly sample from dataset in $i$-th PE
\State Compute the local stochastic gradient
\State Communicate parameters to neighbors
\State Update parameters using~\eqref{eqn:dist_sgd_neigh}
\EndFor
\State Obtain averaged model from all PEs
\end{algorithmic}
\end{algorithm}

\section{Proposed Algorithm: EventGraD}
\label{sec:idea}

In the decentralized algorithm in~\eqref{eqn:dist_sgd_neigh}, the parameters in a PE are exchanged with neighbors in every iteration of the training. This might be a waste of resources since the parameters might not differ a lot in every iteration. Therefore it is possible to relax this requirement of communication with neighbors at every iteration of training. This is the main idea of our algorithm where communication happens \textit{only when necessary} in \textit{events}. 

Our algorithm works as follows - Every PE tracks the changes in the parameters of its model. When the norm of a particular parameter in a PE has changed by some threshold, it is sent to the neighboring PEs. At other iterations, that particular parameter is not sent to the neighbors and the neighbors continue updating their own model using the last received version of that parameter.

\begin{figure}[h]
\centering
\includegraphics[scale=0.4]{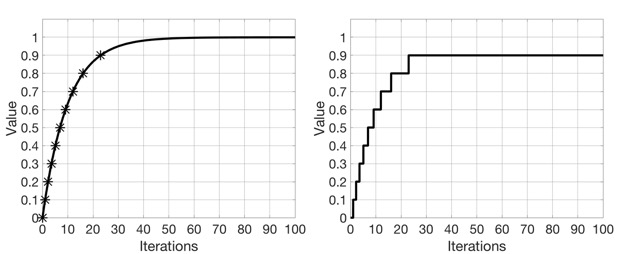}
\caption{Illustration of change in norm of parameters over iterations (taken from~\cite{ghosh2018event}). The left plot shows the norm of the parameter over iterations at the sender. The right plot shows the norm of that corresponding parameter used at the receiver.}
\label{fig:eventdiag}
\end{figure}

Fig~\ref{fig:eventdiag} illustrates this phenomenon. As an example, the left plot shows the evolution of the norm of a parameter over training iterations. When this norm changes by more than a threshold (0.1 in Fig~\ref{fig:eventdiag}) from the norm of the previously communicated values, an event for communication is triggered as marked by an asterisk. The first event of communication is forced to take place at iteration $k=0$ for convenience. The right plot shows the corresponding values that the receiving PE uses when averaging its parameter with the parameter from this corresponding sending PE. 

For mathematically describing the algorithm, let us first define the vector of previously communicated values $\hat{X}_k$ as
\begin{align}
\hat{X}_k := \left[ \hat{x}_{k,1} \quad \cdots \quad  \hat{x}_{k,n} \right] \in \mathbb{R}^{N \times n}.
\end{align}
Note that each $\hat{x}_{k,i}$ is a vector of the norm of $N$ parameters, i.e.,
\begin{align}
\hat{x}_{k,i} = \left[ \hat{x}_{k,i,1} \quad \dots \quad \hat{x}_{k,i,N} \right]^\top \in \mathbb{R}^{N}.
\end{align}

Now the event-triggered condition can be expressed as

\begin{align*}
\hat{x}_{k+1,i,I} =
\begin{cases}
x_{k+1,i,I} & \text{if } \Vert \hat{x}_{k,i,I}-x_{k+1,i,I} \Vert \geq \delta_{k,i,I}\\
\hat{x}_{k,i,I} & \text{if } \Vert \hat{x}_{k,i,I}-x_{k+1,i,I} \Vert < \delta_{k,i,I},\\
\end{cases} 
\end{align*}
where $\delta_{k,i,I}$ is the threshold for the $I$-th parameter in the $i$-th PE at $k$-th iteration. Consequently, the training algorithm gets modified from~\eqref{eqn:dist_sgd_neigh} to
\begin{align}
\label{eqn:dist_sgd_event}
X_{k+1} = \hat{X}_k W - \gamma \partial F\left( \hat{X}_k;\xi_k \right),
\end{align}
which represents our algorithm with event-triggered communication. The pseudo code is specified in Algorithm~\ref{alg:event_ml}. 

\begin{algorithm}
\renewcommand\thealgorithm{\Alph{algorithm}}
\caption{: EventGraD - Event-Triggered Communication in Data Parallel SGD}
\label{alg:event_ml}
\begin{algorithmic}
\For{$k = 0, 1, 2, \dots K-1$}
\State Randomly sample from dataset in $i$-th PE
\State Compute the local stochastic gradient
\For{$I = 1, 2, \dots, N$}
\If {$\Vert \hat{x}_{k,i,I}-x_{k+1,i,I} \Vert \geq \delta_{k,i,I}$}
    \State Communicate parameter to neighbors
\EndIf
\EndFor
\State Update parameters using~\eqref{eqn:dist_sgd_event}
\EndFor
\State Obtain averaged model from all PEs
\end{algorithmic}
\end{algorithm}

Choosing the threshold $\delta_{k,i,I}$ is a design problem. The efficiency of this algorithm depends on selecting appropriate thresholds. The simplest option would be to choose the same value of threshold for all the parameters in all the PEs as was done in~\cite{george2019distributed}. However, selecting the appropriate value would involve a lot of trial and error. Further, when the neural network model changes, the process would have to be repeated all over again. More importantly, the parameters in a model and across different PEs would vary differently and selecting the same threshold for all of them is not desired. 
Instead, it is better to choose a dynamic threshold that is adaptive to the rate of change of the parameters. A metric that is indicative of the rate of change of a parameter is the local slope of the norm of the parameter. Thus we choose the threshold of a parameter based on the local slope of its norm. Whenever an event of communication is triggered, the slope is calculated between the current value and the last communicated value. This slope is multiplied by a horizon $h$ to calculate the threshold as illustrated in Fig~\ref{fig:slope_thres}. This threshold will be kept fixed until the next event is triggered, resulting in calculation of a new threshold. Thus we obtain:
\begin{align}
\label{eqn:thres}
\delta_{k,i,I} = \underbrace{\frac{\Vert \hat{x}_{k,i,I}-x_{k+1,i,I} \Vert}{k - \hat{k}}}_{Slope} \times h,
\end{align}
where $\hat{k}$ is the iteration corresponding to $\hat{x}_{k,i,I}$, i.e., when the last value was communicated.

The intuition behind making the threshold dependent on the slope is to ensure it is chosen according to the trend of evolution of the parameter. This helps on saving communication as much as possible while ensuring that communication does not stop, i.e., happens once in a while.
If a parameter is changing fast, that means that it will satisfy the criterion for communication soon - thus a high threshold (due to the high slope) can be suitable. However, if the parameter is changing slowly, there might a long period before the next communication happens which might slow down convergence of the overall algorithm. Hence the threshold is decreased (due to the low slope) to incentivize communication.

The horizon $h$ is a hyperparameter that is chosen by the user. Its purpose is to serve as a \textit{look-ahead} to calculate the next threshold. It might seem that $h$ requires tuning as well, thereby nullifying its advantages over the static threshold. However, the same value of $h$ can be chosen for the different parameters because the threshold is already modulated by the slope. If the neural network model is changed due to change in the depth, width or type of layers, the threshold will adjust accordingly. Choosing a different dataset where the data follows a different distribution is also likely to change the evolution of the neural network parameters which the adaptive threshold can capture. Thus the adaptive threshold selection mechanism plays a huge role in keeping our algorithm EventGraD portable as much as possible across multiple models and multiple datasets.

\begin{figure}[h]
\centering
\includegraphics[scale=0.38]{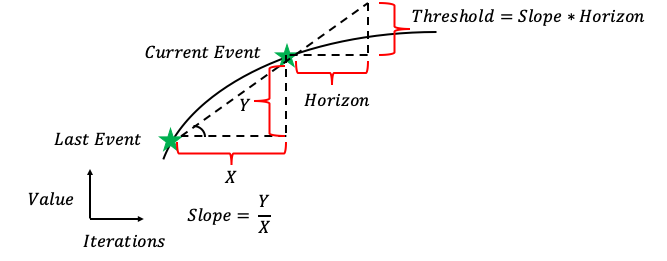}
\caption{Illustration of slope-based adaptive threshold. The right green star denotes the event of current communication while the left green star denotes the event of last communication. The slope is calculated between these two points which is then multiplied by the horizon to obtain the new adaptive threshold.}
\label{fig:slope_thres}
\end{figure}

\section{Analysis}
\label{sec:analysis}

The theoretical convergence properties of the proposed algorithm are studied in this section. Let us consider the error or difference between the last communicated state and the current state as:
\begin{align}
\label{eqn:error_state}
&\epsilon_{k,i,I} = \hat{x}_{k,i,I} - x_{k,i,I} , \\
\implies &\mathcal{E}_{k} = \hat{X}_{k} - X_{k}  ,
\end{align}
where $\mathcal{E}_k = \left[ \epsilon_{k,1} \quad \dots \quad \epsilon_{k,n} \right] \in \mathbb{R}^{N \times n}$. According to our algorithm, the error is bounded by the corresponding threshold as
\begin{align}
\Vert \epsilon_{k,i,I} \Vert \leq \delta_{k,i,I}.
\end{align}
Since $\delta_{k,i,I}$ is different for different $i$ and different $I$, considering the different values for any theoretical analysis seems intractable.
Rather we consider the following assumption:
\begin{assumption}
\label{assum:thres_bound}
The thresholds $\delta_{k,i,I}$ can be bounded by a function dependent on only $k$ as
\begin{align}
\Vert \delta_{k,i,I} \Vert^2 \leq g(k).
\end{align}
\end{assumption}

Assumption~\ref{assum:thres_bound} makes analysis of the convergence properties of the algorithm feasible by considering a bound on thresholds for all parameters in all PEs. Further, we consider the following assumptions that are usually used for analysis of SGD algorithms.

\begin{assumption}
\label{assum:sgd}
The following assumptions hold:

\begin{enumerate}
\item Lipschitz Gradient: All functions $f_i(.)$'s have $L$-Lipschitz Gradients.

\item Spectral Gap: Given the symmetric doubly stochastic matrix $W$, the value $\rho:= \left( \max\right\{\left|\lambda_2(W)\right|,\left|\lambda_n(W) \right|\left\}\right)$ satisfies $\rho<1$.

\item Bounded Variance: The variance of the  stochastic gradient
\begin{align*}
\mathbb{E}_{i \sim \mathcal{U}_{i([n])} }\mathbb{E}_{\xi \sim \mathbb{D}_i} \Vert \nabla F_i(x;\xi)- \nabla f(x) \Vert^2
\end{align*}
is bounded for any $x$ with $i$ sampled uniformly from $\{1,\dots,n\}$ and $\xi$ from the distribution $\mathbb{D}_i$. That is, there are constants $\sigma$ and $\varsigma$, such that:
\begin{align*}
\mathbb{E}_{\xi \sim \mathcal{D}_i} \Vert \nabla F_i(x;\xi)-\nabla f_i(x) \Vert^2\leq \sigma^2,\forall i, \forall x, \\
\mathbb{E}_{i \sim \mathcal{U}_{i([n])}} \Vert \nabla f_i(x)-\nabla f(x) \Vert^2\leq \varsigma^2,\forall x.
\end{align*}
\item We start from $X_0=0$ without loss of generality.
\end{enumerate}
\end{assumption}

Let
\begin{align}
&C_1 = \left(\frac{1-\gamma}{2}-\frac{72\gamma^3}{C_2\left(1-\sqrt{\rho}\right)^2}L^2\right), \;
C_2 = \left(1-\frac{36\gamma^2}{\left(1-\sqrt{\rho}\right)^2}nL^2\right) \nonumber \\
&G(K) = \sum_{k=0}^K g(k), \;
G_{1/2}(K) = \sum_{k=0}^K \sqrt{g(k)}
\end{align}
\begin{theorem}
\label{thm:conv}
Considering the assumptions, we obtain the following convergence rate for the algorithm
\begin{align}
\begin{split}
&\quad \frac{C_1}{K}\sum_{k=0}^{K-1}{\mathbb{E}\left\Vert\nabla f\left(\frac{X_{k}\textbf{1}_n}{n}\right)\right\Vert^2}+ \frac{\gamma-\gamma^2 L}{2K}\sum_{k=0}^{K-1}{\mathbb{E}\left\Vert \frac{\partial f(\hat{X}_k)\textbf{1}_n}{n}\right\Vert^2}\\
& \quad \leq \frac{f(0)-f^*}{K}
+ \frac{\gamma^2 L\sigma^2}{2n} \\
& \quad +
\left(12C_2^{-1}\gamma^3nL^2\left(2L^2+1\right)+\frac{3\gamma L^2+L+1}{2K}+\frac{72\gamma^3L^4}{KC_2\left(1-\sqrt{\rho}\right)^2}\right)G(K-1)\\
& \quad +C_2^{-1}\gamma\rho L^2G_{1/2}^2(K-1)+\frac{2n\gamma^3\sigma^2L^2}{C_2(1-\rho)}+\frac{18n\gamma^3\varsigma^2L^2}{C_2\left(1-\sqrt{\rho}\right)^2}\\
\end{split}
\end{align}
\end{theorem}
\begin{proof}
Provided in appendix.
\end{proof}

Theorem~\ref{thm:conv} represents the convergence of the average of the models in all PEs. In order to obtain a closer result, we consider an appropriate learning rate and then state the following corollary:

Let
\begin{align}
C_3=\frac{\left(1-\sqrt{\rho}\right)^2\left(2L^2+1\right)}{6\rho L^2}, \; C_4=\frac{7L^2+L+1}{2}
\end{align}

\begin{corollary}
\label{corr:conv}
Under the same assumptions as in Theorem~\ref{thm:conv}, if we set $\gamma = \frac{1}{2\rho L^2\sqrt{K}+\sigma\sqrt{K/n}}$, we have the following convergence rate:
\begin{align*}
\begin{split}
\frac{1}{K}\sum_{k=0}^{K-1}{\mathbb{E}\left\Vert\nabla f\left(\frac{X_{k}\textbf{1}_n}{n}\right)\right\Vert^2} \leq 
\left(2(f(0)-f^*)+L\right)\left(\frac{1}{K}+\frac{1}{\sqrt{Kn}}\right) \\
+
\left(\frac{2C_3}{\sqrt{K}}+\frac{2C_4}{K}\right)G(K-1)+ 
\frac{2}{\sqrt{K}}G_{1/2}^2(K-1)\\
\end{split}
\end{align*}
if the total number of iterations $K$ is large enough, in particular,
\begin{align*}
&K \geq \frac{4n^3L^2}{\sigma^3\left(f(0)-f^*+L/2\right)}\left(\frac{\sigma^2}{\left(1-\rho\right)}+\frac{9\varsigma^2}{\left(1-\sqrt{\rho}\right)^2}\right), \text{and} \\
&K \geq \frac{72 L^2 n^2}{\sigma^2 (1 - \sqrt{\rho})^2}, \text{and}
\\
&K \geq \left( \frac{\sqrt{n}(L+1)}{2 \rho L^2 \sqrt{n} + \sigma} \right)^2
\end{align*}
\end{corollary}
\begin{proof}
Provided in appendix.
\end{proof}

Corollary~\ref{corr:conv} shows that the bound on the convergence rate is dependent on the threshold related terms $G(K)$ and $G_{1/2}(K)$. When $K$ is large enough, the $\frac{1}{K}$ terms will decay faster than the $\frac{1}{\sqrt{K}}$ terms and therefore the convergence rate is of the order $\mathcal{O}\left(\frac{1}{\sqrt{Kn}} + \frac{G(K-1)}{\sqrt{K}} + \frac{G_{1/2}^2(K-1)}{\sqrt{K}}\right)$. 

Note that a threshold of $0$ reduces to the regular algorithm in~\cite{lian2017can}. Thus with $g(k) = 0$, we obtain $G(k)=0$ and $G_{1/2}(k)=0$ and hence the rate of convergence reduces to $\mathcal{O}(\frac{1}{K}+\frac{1}{\sqrt{Kn}})$ which is consistent with~\cite{lian2017can}. Now we provide a more concrete bound by choosing $g(k)$ according to a popular event-triggered threshold specified in~\cite{seyboth2013event}.

\begin{corollary}
\label{corr:geo_gk}
If $g(k)$ is chosen of the form $g(k) = \alpha \beta^k$ where $\alpha, \beta$ are appropriate constants and $0 < \beta < 1$, then the rate of convergence is of the order $\mathcal{O}\left(\frac{1}{\sqrt{Kn}} + \frac{1}{\sqrt{K}}\right)$.
\end{corollary}
\begin{proof}
We obtain $G(K-1) = \alpha \left( \frac{1 - \beta^{K-1}}{1 - \beta} \right)$ and $G_{1/2}^2(K) = \alpha \left( \frac{1 - \sqrt{\beta}^{K-1}}{1-\sqrt{\beta}} \right)^2$. The corollary then follows by noting that $\left(\frac{2C_3}{\sqrt{K}}+\frac{2C_4}{K}\right)\alpha \left( \frac{1 - \beta^{K-1}}{1 - \beta} \right) \sim \mathcal{O}\left(\frac{1}{\sqrt{K}}\right)$ and $\frac{4}{\sqrt{K}}\alpha \left( \frac{1 - \sqrt{\beta}^{K-1}}{1-\sqrt{\beta}} \right)^2 \sim \mathcal{O}\left(\frac{1}{\sqrt{K}}\right)$ when $K$ is sufficiently large.
\end{proof}

\section{Implementation}
\label{sec:impl}

There are a lot of popular frameworks for machine learning like PyTorch, TensorFlow, CNTK, etc. Almost all of these frameworks support parallel or distributed training. TensorFlow follows the parameter server approach for parallelization. PyTorch provides a module called DistributedDataParallel that implements AllReduce based training. Horovod is another framework developed by Uber that implements an optimized AllReduce algorithm. However, none of these frameworks provide native support for the training involving averaging with just neighbors. Hence we decided to implement the proposed algorithms without using any of the distributed modules in these frameworks.

We use PyTorch and MPI for our implementation. First, we point out why one-sided communication or remote memory access is necessary. Usually communication in high performance computing networks is two-sided. In other words, the sending PE starts the communication of a message by invoking a MPI\_Send operation and then the receiving PE completes the communication and receives the message by invoking a MPI\_Recv operation~\cite{gropp1999using}. In our event-triggered communication algorithm, the events for communication are dependent on the change in values of the parameters of the sender which is a local phenomenon. Thus, when an event is triggered in the sending PE, it can issue a MPI\_Send operation. However, since the intended receiving PE is not aware of when the event is triggered at the sender, it does not know when to issue a MPI\_Recv operation. So two-sided communication using MPI\_Send and MPI\_Recv cannot be used for our algorithm.

Hence we select one-sided communication for our purpose. In one-sided communication, only the sending PE has to know all the parameters of the message for both the sending and receiving side and can remotely write to a portion of the memory of the receiver without the receiver's involvement - hence the alternate name of Remote Memory Access~\cite{gropp2014using}. That region of memory in the receiver is called \textit{window} and can be publicly accessed. In our case, it is used to store the model parameters from the neighbors. So when an event for communication is triggered in the sending PE, it uses MPI\_Put to write its model parameters directly into the window of the corresponding neighbor PE. An illustration of one-sided vs two-sided communication is provided in Fig~\ref{fig:two-vs-onesided}.

\begin{figure}[h]
\centering
\includegraphics[scale=0.4]{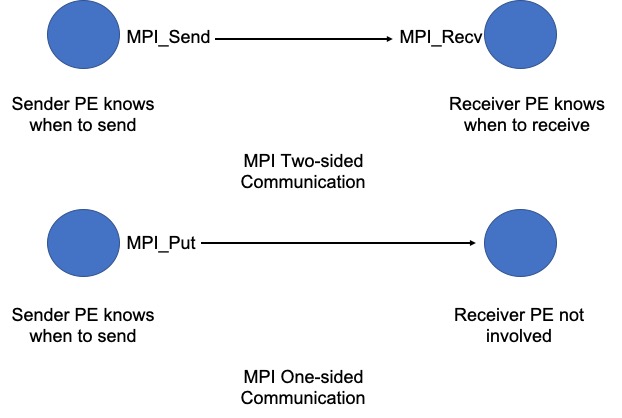}
\caption{Illustration of the difference between two-sided and one-sided communication.}
\label{fig:two-vs-onesided}
\end{figure}

It is worth noting that PyTorch does not support one-sided communication at this point. Recently, PyTorch released a C++ frontend called Libtorch which can be integrated with traditional C++ MPI implementations. Further, the C++ frontend is more suitable for HPC environments unlike the Python frontend. Hence we combine the neural network training functionalities of Libtorch with communication routines in MPI to implement our algorithm. For further details on our implementation, the reader is referred to~\cite{ghosh2020eventgrad}.

\section{Results}
\label{sec:results}

We perform experiments to evaluate the performance of our algorithm. All our simulations are done on CPUs. We use an HPC cluster of nodes with each node having 2 CPU Sockets of AMD's EPYC 24-core 2.3 GHz processor and 128 GB RAM per node. The cluster uses Mellanox EDR interconnect. The MPI library chosen is Open MPI 4.0.1 compiled with gcc 8.3.0. The version of Libtorch used is 1.5.0. We conduct our experiments on the CIFAR-10 dataset. We choose the residual neural network commonly used for training on CIFAR-10~\cite{he2016deep}. Our simulations use the ResNet-18 configuration. For training this network, a learning rate of $0.01$ is used with cross-entropy as the loss function and a mini-batch size of $256$. The value of horizon $h$ used to calculate the threshold in~\eqref{eqn:thres} is taken to be $1$. Note that we performed experiments on the MNIST dataset in our previous work~\cite{ghosh2020eventgrad}.

To illustrate our adaptive event-triggered threshold selection scheme, we look at how the norm of the parameters, i.e., the weights and biases, in the ResNet-18 model change with iterations. Note that the ResNet-18 model has $86$ parameters, all of which cannot be shown in this paper due to space constraints. Therefore we show a few parameters which vary in their style of evolution in Fig~\ref{fig:param_norm}. After few initial oscillations, the change of the values is gradual which suggests that not all parameters need to be communicated at every iteration. This paves the way for saving on communication of messages by event-triggered communication. The corresponding threshold evolution as calculated by the equation in~\eqref{eqn:thres} is shown in Fig~\ref{fig:thres_norm}. Since the threshold is proportional to the local slope, we see in parameter $1$ and $3$ that higher slopes during the early iterations of training lead to higher thresholds followed by a decrease in threshold due to decrease in slope. For parameters $2$ and $5$, which stay relatively flat, the threshold also follows a flat trend. It is important to note that since every parameter changes differently, their thresholds have to be chosen accordingly.

\begin{figure}[h]
\centering
\includegraphics[scale=0.15]{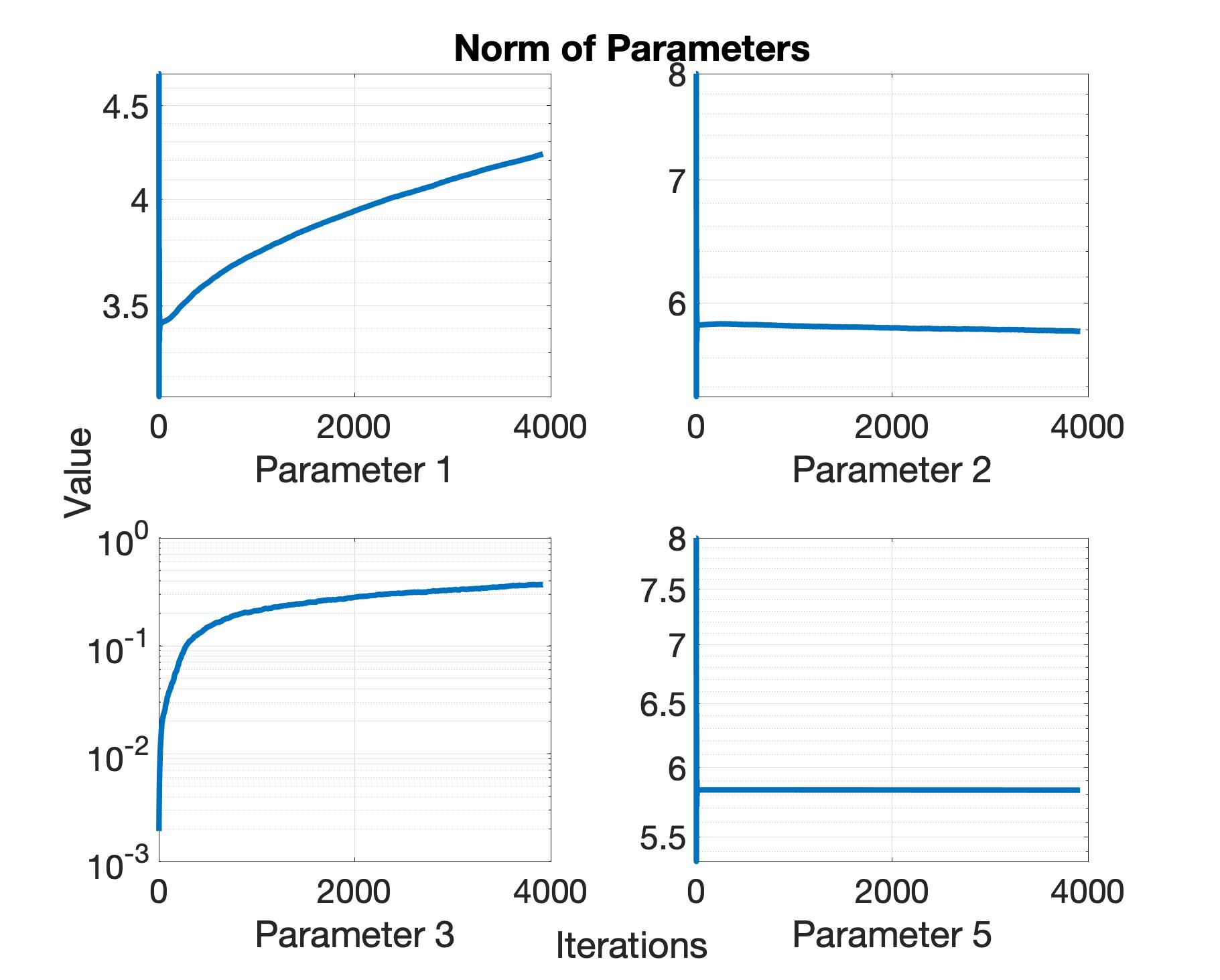}
\caption{Plot showing the evolution of the norm of the parameters of the neural network in a certain PE. Note that the parameters may vary in their trend of evolution.}
\label{fig:param_norm}
\end{figure}

\begin{figure}[h]
\centering
\includegraphics[scale=0.15]{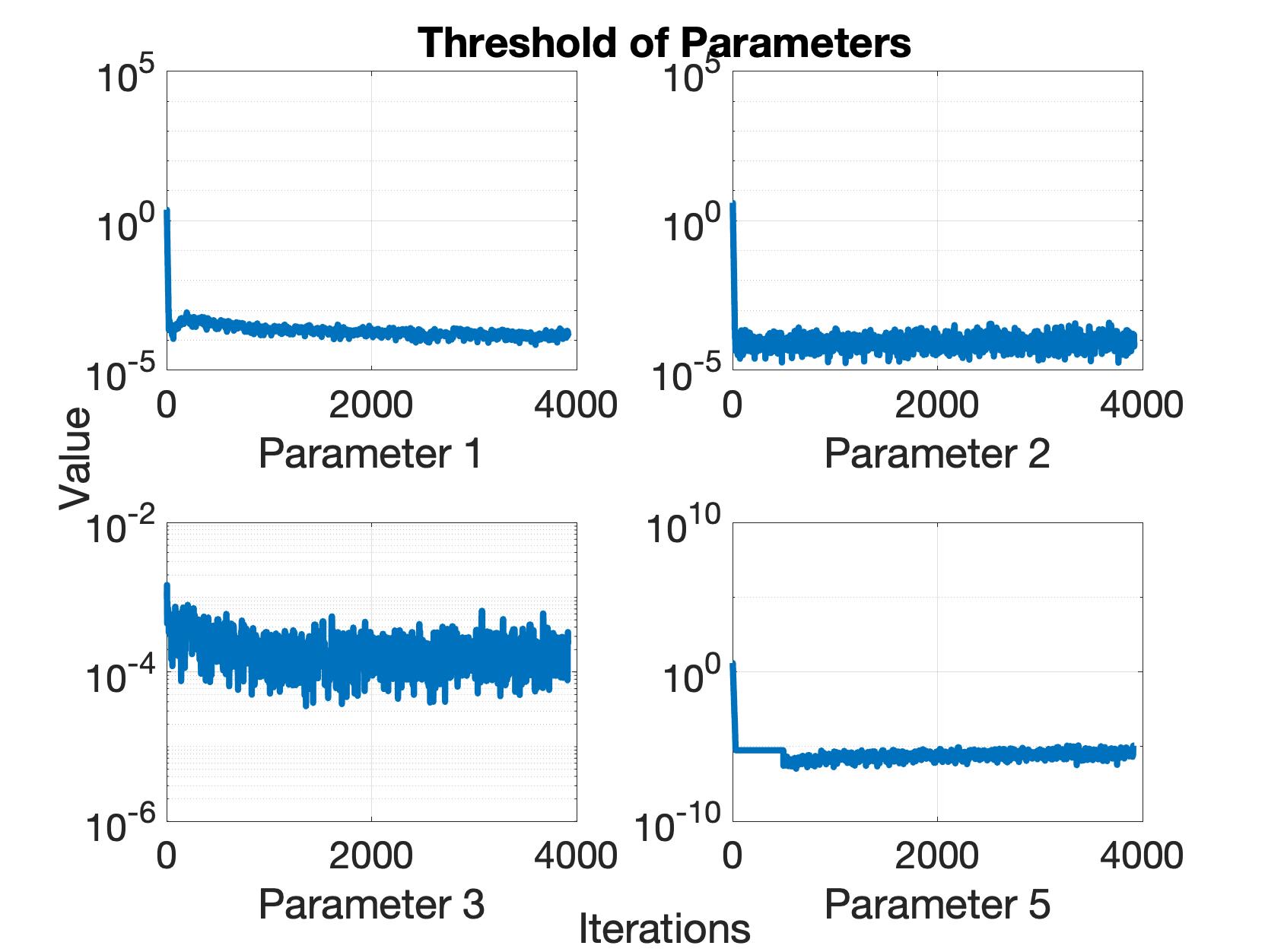}
\caption{Plot showing the adaptive threshold of the parameters shown in Fig~\ref{fig:param_norm}. The trend of the threshold is adaptive to the trend of evolution of the corresponding parameter.}
\label{fig:thres_norm}
\end{figure}

The thresholds in Fig~\ref{fig:thres_norm} have an oscillatory behavior. This is due to the fact that often the parameters in a neural network have local minor oscillations because of the nature of the stochastic gradient descent algorithm. The parameters in Fig~\ref{fig:param_norm} have these local oscillations, however they are not prominently visible due to the higher scale of the plot. Further, the stochastic nature of the MPI one-sided implementation of the algorithm amplifies the oscillations. It is desired that the threshold reflect the aggregate trend of evolution in the parameter and not the local oscillations. In order to solve this issue, the sender can keep a history of multiple previously communicated events instead of just one previous event. Then the average slope is calculated which is the mean of the slopes between two consecutive events in that history. This average slope is then multiplied by the horizon to obtain the threshold. The length of the history is a hyperparameter which is similar in notion to the length of a moving average filter. The higher the length, the smoother the trend but at the cost of increased computational complexity.

Having described details of selecting the threshold, we now look at the experimental convergence properties of the algorithm. We compare our event-triggered communication algorithm proposed in Algorithm~\ref{alg:event_ml} with respect to the regular communication algorithm in Algorithm~\ref{alg:regular_ml} from~\cite{lian2017can}. Fig~\ref{fig:loss_vs_epoch} shows the loss function over epochs for both these algorithms, each repeated for $10$ different runs shown by the errorbars. Note that an epoch refers to processing the entire dataset allotted to a PE once while an iteration refers to processing a batch once. Thus one epoch has multiple iterations which depends on the size of the batch. From Fig~\ref{fig:loss_vs_epoch}, we see that the decay in loss function seems similar for both the algorithms, indicating that they have similar speed of convergence. It is important to observe that the theoretical results in Section~\ref{sec:analysis} deal with a bound on convergence of the average of the parameters in all the PEs whereas the plot in Fig~\ref{fig:loss_vs_epoch} is concerned with experimental convergence of the loss function.

\begin{figure}[h]
\centering
\includegraphics[scale=0.15]{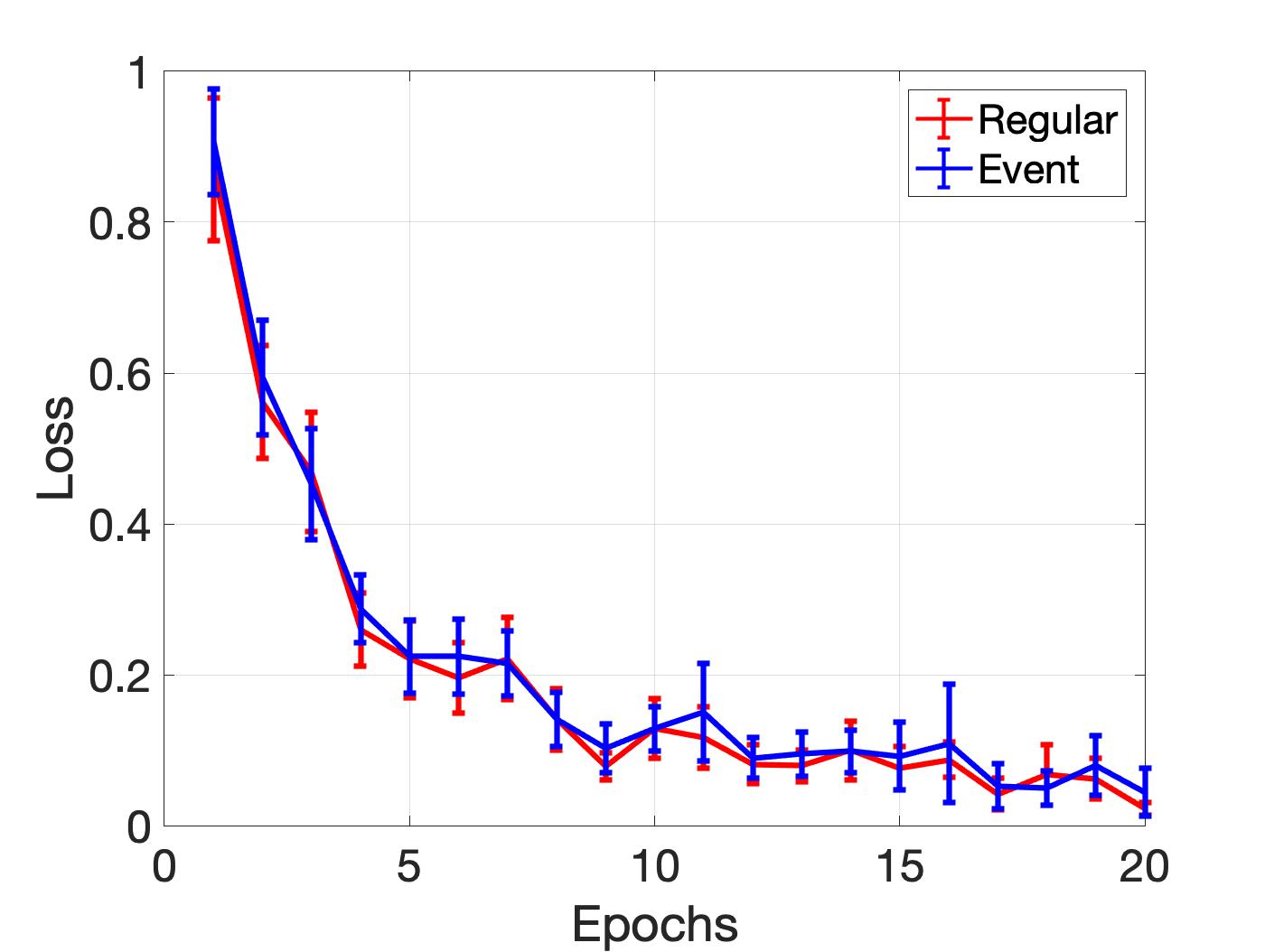}
\caption{Plot showing the loss function over epochs. The experiments have been repeated $10$ times to account for variations which have been considered in the errorbars. It is seen that the event-triggered communication algorithm has an experimental rate of convergence similar to that of the regular communication algorithm.}
\label{fig:loss_vs_epoch}
\end{figure}

After demonstrating similar rate of convergence, we focus on the main advantage of the event-triggered algorithm over the regular algorithm - reduction in the number of messages communicated while attaining similar accuracy. The reduction in messages is quantified by the percentage of messages of the regular algorithm that is sent in the event-triggered algorithm. Table~\ref{table:msgs} states the accuracy of regular and event-triggered communication as well as the percentage of messages in event-triggered communication after training for 20 epochs. 

\begin{table}[h!]
\begin{center}
\caption{Comparison of Regular Communication vs Event-Triggered Communication after 20 epochs. The Event-Triggered Communication algorithm drastically reduces the number of messages to be communicated by around $60\%$ while maintaining similar accuracy.}
\label{table:msgs}
\resizebox{\columnwidth}{!}{
 \begin{tabular}{| c | c | c | c |} 
 \hline 
 \textbf{Number of PEs} &
 \textbf{Regular Accuracy} &
 \textbf{Event-Triggered Accuracy} &
 \textbf{Percentage of Messages} \\ 
 \hline
 4 & 86.5 & 87 & 43.24 \\
 \hline
 8 & 86.3 & 86.2 & 42.98 \\
 \hline
 16 & 84.9 & 84.2 & 45.91
 \\
 \hline
 32 & 82.5 & 81.9 & 44.89
 \\
 \hline
\end{tabular}
}
\end{center}
\end{table}

We see that the event-triggered communication algorithm exchanges approximately $40\%$ of the messages of the regular (baseline) communication algorithm while achieving similar accuracy. As a reminder, the regular algorithm is the D-PSGD algorithm specified in~\cite{lian2017can} that is also suitable for the asynchronous decentralized environment that we consider. In other words, our event-triggered algorithm saves around $60\%$ of the messages as compared to the baseline algorithm while maintaining similar accuracy, thus alleviating the communication overhead. Note that the accuracy of both the regular and event-triggered algorithms decrease as the number of PEs increase. This is because as the ring of PEs get larger, messages comprising of neural network parameters require more hops to propagate through the entire ring. Hence, given same number of epochs, the larger ring comprising of more PEs will have lesser accuracy. If the algorithm is run for more epochs on more PEs, the accuracy will not degrade. Additionally, if the number of neighbors of each PE is increased, information can flow sooner, resulting in more accuracy.

A noteworthy feature of our proposed algorithm is that it is complementary to other algorithms for reduced communication that have been proposed in the literature. In other words, our algorithm can be combined with these algorithms. For instance, we can apply the techniques of quantization and sparsification on top of event-triggered communication to get even more savings in communication. It is important to clarify that most of the existing works in literature apply quantization and sparsification in the parameter server or AllReduce architecture~\cite{seide20141-bit,alistarh2018convergence} which is different from the decentralized reduction with just neighbors scenario that we deal with. 
However, to demonstrate how these approaches can be extended to our decentralized scenario and combined with the event-triggered approach, we focus on the sparsification method of Top-K. Specifically, when an event is triggered, we send just the Top-K percentage of the elements in a parameter, i.e., the weight matrix or the bias vector. Note that for a Top-K percent value of K, 2K percent of messages is being sent because the indices of the Top-K percent elements have to be sent in addition to their values.

Table~\ref{table:event_sparse} shows the results of combining Top-K percent sparsification with event-triggered communication. Before we compare the results in Table~\ref{table:event_sparse} with Table~\ref{table:msgs}, we note some important points. Firstly, even though all the simulation details of the event-triggered communication are kept the same between Table~\ref{table:msgs} and Table~\ref{table:event_sparse}, the percentage of messages sent are different between them. This is because sending just the Top-K elements of a parameter changes the overall evolution of the neural network which in turn leads to different adaptive thresholds and hence different sequence of events. Secondly, we have to consider the percentage of overall communication for Top-K sparsification in contrast to percentage of messages considered in Table~\ref{table:msgs}. This is due to the fact that the objective of Top-K sparsification is to reduce the size of each message sent. Note that in Table~\ref{table:msgs}, the percentage of overall communication is equivalent to the percentage of messages mentioned since entire parameters are sent during events. However, in the case of Top-K in Table~\ref{table:event_sparse}, the percentage of overall communication is 2K $\%$ of the percentage of messages sent. Now we see that the accuracy in Table~\ref{table:event_sparse} remains almost similar to that of Table~\ref{table:msgs} but the overall communication is approximately $7\%$ of that of the regular (baseline) communication algorithm. This is in contrast to the overall communication of around $40\%$ in the event-triggered algorithm in Table~\ref{table:msgs} with respect to the baseline. Thus Top-K sparsification combined with event-triggered communication requires around $\frac{1}{6}$-th of the communication required in just event-triggered communication while maintaining similar accuracy.

\begin{table}[h!]
\begin{center}
\caption{Combining Top-K$\%$ Sparsification with Event-Triggered (ET) Communication for K$=10$. The percent of communication is 2K$=20\%$ of the percent of messages. It is seen that the communication required here is around $\frac{1}{6}$-th of that in Event-Triggered (ET) communication without sparsification while maintaining similar accuracy as in Table~\ref{table:msgs}.}
\label{table:event_sparse}
\resizebox{\columnwidth}{!}{
 \begin{tabular}{| c | c | c | c |} 
 \hline 
 \textbf{Number of PEs} &
 \textbf{Sparse ET Accuracy} &
 \textbf{Percent of Messages} &
 \textbf{Percent of communication}\\ 
 \hline
 4 & 85.4 & 36.6 & 7.3 \\
 \hline
 8 & 85.26 & 37.7 & 7.5 \\
 \hline
 16 & 83.09 & 37.7 & 7.5
 \\
 \hline
\end{tabular}
}
\end{center}
\end{table}

\section{Conclusion}
\label{sec:conc}

This paper introduces a novel algorithm that reduces communication in parallel training of neural networks. The proposed EventGraD algorithm communicates the model parameters in events only when the value of the parameter changes by a threshold. The choice of the threshold for triggering events is chosen adaptively based on the slope of the parameter values. The algorithm can be applied to different neural network configurations and different datasets. An asymptotic bound on the rate of convergence is provided. The challenges of implementing this algorithm in a high performance computing cluster, such as the requirement of advanced communication protocols and libraries, are discussed. Experiments on the CIFAR-10 dataset show the superior communication performance of the algorithm while maintaining the same accuracy.  

\section{Acknowledgements}

This research was supported in part by the University of Notre Dame Center for Research Computing through its computing resources. The work of the authors was supported in part by ARO grant W911NF1910483, DARPA FA8750-20-2-0502 and NSF CBET-1953090. 

\section{Appendix}

The proof for the theoretical results in this paper are provided here. First we state some necessary lemmas. Lemma~\ref{lemma:rhok_bound} and Lemma~\ref{lemma:break_df} are reproduced from~\cite{lian2017can}.

\begin{lemma}
\label{lemma:rhok_bound}
Using Assumption~\ref{assum:sgd}, we obtain
\begin{align*}
\left\Vert \frac{\textbf{1}_n}{n} - W^k e_i\right\Vert^2 \leq \rho^k, \forall i \in {1,2,\dots,n}, k \in \mathbb{N}
\end{align*}
\end{lemma}
\begin{proof}
Let $W^{\infty}:= lim_{k \rightarrow \infty} W^k$. Because of the assumptions, we get $\frac{\mathbf{1}_n}{n} = W^{\infty}e_i \forall i$ since $W$ is doubly stochastic and $\rho < 1$. Thus
\begin{align*}
\left\Vert \frac{\textbf{1}_n}{n} - W^k e_i\right\Vert^2 = &\left\Vert(W^{\infty} - W^k)e_i\right\Vert^2 \\
&\leq \left\Vert W^{\infty} - W^k\right\Vert^2 \left\Vert e_i\right\Vert^2 \\
&= \left\Vert W^{\infty} - W^k\right\Vert^2 \\
&\leq \rho^k.
\end{align*}
\end{proof}

\begin{lemma}
\label{lemma:break_df}
Under Assumption~\ref{assum:sgd}, the following holds:
\begin{align*}
\mathbb{E}\left\Vert \partial f(X_j)\right\Vert^2 \leq \sum_{h=1}^n 3\mathbb{E}L^2 \left\Vert \frac{\sum_{i'=1}^n x_{j,i'}}{n} - x_{j,h} \right\Vert^2 + 3n\varsigma^2 + 3\mathbb{E}\left\Vert \nabla f\left( \frac{X_j \mathbf{1}_n}{n}\right)\mathbf{1}_n^\top \right\Vert^2, \forall j
\end{align*}
\end{lemma}
\begin{proof}
The term $\mathbb{E}\left\Vert \partial f(X_j)\right\Vert^2$ is bounded as follows:
\begin{align*}
&\mathbb{E}\left\Vert \partial f(X_j)\right\Vert^2 \\
&\leq 3\mathbb{E}\left\Vert \partial f(X_j) - \partial f \left( \frac{X_j \mathbf{1}_n}{n}\mathbf{1}_n^\top \right)\right\Vert^2 \\
&\quad + 3\mathbb{E}\left\Vert \partial f \left( \frac{X_j \mathbf{1}_n}{n}\mathbf{1}_n^\top \right) - \nabla f\left( \frac{X_j \mathbf{1}_n}{n}\right)\mathbf{1}_n^\top \right\Vert^2 \\
&\quad + 3\mathbb{E} \left\Vert \nabla f  \left( \frac{X_j \mathbf{1}_n}{n}\right)\mathbf{1}_n^\top \right\Vert^2 \\
&\leq 3\mathbb{E}\left\Vert \partial f(X_j) - \partial f \left( \frac{X_j \mathbf{1}_n}{n}\mathbf{1}_n^\top \right)\right\Vert_{F}^2 \\
&\quad + 3n\varsigma^2 + 3\mathbb{E} \left\Vert \nabla f \left(\frac{X_j \mathbf{1}_n}{n}\right)\mathbf{1}_n^\top \right\Vert^2 \\
&\leq \sum_{h=1}^n 3\mathbb{E}L^2 \left\Vert \frac{\sum_{i'=1}^n x_{j,i'}}{n} - x_{j,h} \right\Vert^2 + 3n\varsigma^2 + 3\mathbb{E} \left\Vert \nabla f \left(\frac{X_j \mathbf{1}_n}{n}\right)\mathbf{1}_n^\top \right\Vert^2.
\end{align*}
\end{proof}

\begin{lemma}
\label{lemma:2a2b_bound}
For any two vectors $a,b$, the following is satisfied
\begin{align}
\Vert a + b \Vert^2 \leq 2 \Vert a \Vert^2 + 2 \Vert b \Vert^2
\end{align}
\end{lemma}
\begin{proof}
We start with $\Vert a + b \Vert^2 = \Vert a \Vert^2 + 2 \left< a,b \right> + \Vert b \Vert^2$. Now 
\begin{align*}
\left<a,b\right> < \sqrt{\Vert a \Vert^2 \Vert b \Vert^2} < \frac{\Vert a \Vert^2 + \Vert b \Vert^2}{2}
\end{align*}
where the first is Cauchy-Schwarz inequality and the second is the geometric mean-arithmetic mean inequality. Substituting the above, the inequality follows.
\end{proof}

\begin{proof}[\textbf{Proof to Theorem~\ref{thm:conv}}]
We begin with
$f\left(\frac{X_{k+1}\textbf{1}_n}{n}\right)$:
\begin{align}
\label{eqn:start}
&\quad \mathbb{E}f\left(\frac{X_{k+1}\textbf{1}_n}{n}\right) \nonumber \\
&\quad =\mathbb{E}f\left(\frac{\hat{X}_{k}W\textbf{1}_n}{n}-\gamma\frac{\partial F\left(\hat{X}_{k};\xi_k\right)\textbf{1}_n}{n}\right) \nonumber \\
&\quad =\mathbb{E}f\left(\frac{\hat{X}_{k}\textbf{1}_n}{n}-\gamma\frac{\partial F\left(\hat{X}_{k};\xi_k\right)\textbf{1}_n}{n}\right) \nonumber \\
&\quad \leq \mathbb{E}f\left(\frac{\hat{X}_{k}\textbf{1}_n}{n}\right)-\gamma\mathbb{E}\left<\nabla f\left(\frac{\hat{X}_{k}\textbf{1}_n}{n}\right),\frac{\partial f \left(\hat{X}_{k}\right)\textbf{1}_n}{n}\right>+\frac{\gamma^2L}{2}\mathbb{E}\left\Vert \sum_{i=1}^{n}{\frac{\nabla F_i \left( \hat{x}_{k,i};\xi_{k,i}\right)}{n}}\right\Vert^2
\end{align}
where the previous step comes from the general Lipschitz property $f(y)<f(x)+\nabla f(x)^T(y-x)+\frac{L}{2}\left\Vert y-x\right\Vert^2$. The last term above is the second order moment of $\sum_{i=1}^n \frac{\nabla F_i \left( \hat{x}_{k,i};\xi_{k,i}\right)}{n}$. Now we can write
\begin{align}
\label{eqn:varmeanrel}
\mathbb{E} \left\Vert \sum_{i=1}^n \frac{\nabla F_i \left( \hat{x}_{k,i};\xi_{k,i}\right)}{n} \right\Vert^2 = \mathbb{E} \left\Vert \sum_{i=1}^n \frac{\nabla F_i \left( \hat{x}_{k,i};\xi_{k,i}\right) - \nabla f_i\left(\hat{x}_{k,i}\right)}{n} \right\Vert^2 + \mathbb{E}\left\Vert \sum_{i=1}^{n}{\frac{\nabla f_i\left(\hat{x}_{k,i}\right)}{n}}\right\Vert^2.
\end{align}

Applying~\eqref{eqn:varmeanrel} in~\eqref{eqn:start}, we obtain:
\begin{align}
\label{eqn:expanded}
&\mathbb{E}f\left(\frac{\hat{X}_{k}\textbf{1}_n}{n}\right)-\gamma\mathbb{E}\left<\nabla f\left(\frac{\hat{X}_{k}\textbf{1}_n}{n}\right),\frac{\partial f \left(\hat{X}_{k}\right)\textbf{1}_n}{n}\right>+\frac{\gamma^{2}L}{2}\mathbb{E}\left\Vert \sum_{i=1}^{n}{\frac{\nabla F_i \left( \hat{x}_{k,i};\xi_{k,i}\right)-\nabla f_i\left(\hat{x}_{k,i}\right)}{n}}\right\Vert^2 \nonumber \\
&\quad + \frac{\gamma^{2}L}{2}\mathbb{E}\left\Vert \sum_{i=1}^{n}{\frac{\nabla f_i\left(\hat{x}_{k,i}\right)}{n}}\right\Vert^2
\end{align}
For the second last term, we can show that
\begin{align}
\label{eqn:exp_linear}
\mathbb{E}\left\Vert \sum_{i=1}^{n}{\frac{\nabla F_i \left( \hat{x}_{k,i};\xi_{k,i}\right)-\nabla f_i\left(\hat{x}_{k,i}\right)}{n}}\right\Vert^2 
= 
\frac{1}{n^2}\sum_{i=1}^{n}\mathbb{E}\left\Vert \left(\nabla F_i \left( \hat{x}_{k,i};\xi_{k,i}\right)-\nabla f_i\left(\hat{x}_{k,i}\right)\right)\right\Vert^2.    
\end{align}
Applying~\eqref{eqn:exp_linear} into~\eqref{eqn:expanded},
\begin{align}
&\mathbb{E}f\left(\frac{\hat{X}_{k}\textbf{1}_n}{n}\right)
-
\gamma\mathbb{E}\left<\nabla f\left(\frac{\hat{X}_{k}\textbf{1}_n}{n}\right),\frac{\partial f \left(\hat{X}_{k}\right)\textbf{1}_n}{n}\right>
+
\frac{\gamma^2L}{2n^2}\sum_{i=1}^{n}\mathbb{E}\left\Vert \left(\nabla F_i \left( \hat{x}_{k,i};\xi_{k,i}\right)-\nabla f_i\left(\hat{x}_{k,i}\right)\right)\right\Vert^2 \nonumber
\\ +
&\frac{\gamma^2L}{2}\mathbb{E}\left\Vert \sum_{i=1}^{n}{\frac{\nabla f_i\left(\hat{x}_{k,i}\right)}{n}}\right\Vert^2
\leq \\
&\mathbb{E}f\left(\frac{\hat{X}_{k}\textbf{1}_n}{n}\right)
-
\gamma\mathbb{E}\left<\nabla f\left(\frac{\hat{X}_{k}\textbf{1}_n}{n}\right),\frac{\partial f \left(\hat{X}_{k}\right)\textbf{1}_n}{n}\right>
+
\frac{\gamma^2L\sigma^2}{2n} \nonumber
\\ +
& \frac{\gamma^2L}{2}\mathbb{E}\left\Vert \sum_{i=1}^{n}{\frac{\nabla f_i\left(\hat{x}_{k,i}\right)}{n}}\right\Vert^2
\end{align}

Using the property $\left<a,b\right>=\frac{1}{2}\left(\Vert a \Vert^2+\Vert b \Vert^2-\Vert a-b \Vert^2\right)$, we can rewrite the above as:
\begin{align}
\label{eqn:proof_T1define}
&\mathbb{E}f\left(\frac{\hat{X}_{k}\textbf{1}_n}{n}\right)
+
\frac{\gamma^2L\sigma^2}{2n}
-
\frac{\gamma}{2}\mathbb{E}\left\Vert\nabla f\left(\frac{\hat{X}_{k}\textbf{1}_n}{n}\right)\right\Vert^2
+
\frac{\gamma^2L-\gamma}{2}\mathbb{E}\left\Vert \sum_{i=1}^{n}{\frac{\nabla f_i\left(\hat{x}_{k,i}\right)}{n}}\right\Vert^2 \nonumber
\\ +
&\frac{\gamma}{2}\underbrace{{\mathbb{E}\left\Vert\nabla f\left(\frac{\hat{X}_{k}\textbf{1}_n}{n}\right)-\frac{\partial f \left(\hat{X}_{k}\right)\textbf{1}_n}{n}\right\Vert^2}}_{:=T_1}
\end{align}

Using the Lipschitz property again, we bound the first term as:
\begin{equation}
\begin{split}\label{lipschitzbound}
& \quad \mathbb{E}f\left(\frac{\hat{X}_{k}\textbf{1}_n}{n}\right)\leq \mathbb{E}f\left(\frac{X_{k}\textbf{1}_n}{n}\right)+\mathbb{E}\left<\nabla f\left(\frac{X_{k}\textbf{1}_n}{n}\right),\frac{\mathcal{E}_{k}1^n}{n}\right>+\frac{L}{2}\left\Vert\frac{\mathcal{E}_{k}1^n}{n}\right\Vert^2\\
& \quad \leq \mathbb{E}f\left(\frac{X_{k}\textbf{1}_n}{n}\right)+\frac{1}{2}\mathbb{E}\left\Vert \nabla f\left(\frac{X_{k}\textbf{1}_n}{n}\right) \right\Vert^2+\frac{1}{2}\left\Vert\frac{\mathcal{E}_{k}1^n}{n}\right\Vert^2+\frac{L}{2}\left\Vert\frac{\mathcal{E}_{k}1^n}{n}\right\Vert^2\\
& \quad = \mathbb{E}f\left(\frac{X_{k}\textbf{1}_n}{n}\right)+\frac{1}{2}\mathbb{E}\left\Vert \nabla f\left(\frac{X_{k}\textbf{1}_n}{n}\right) \right\Vert^2+\frac{L+1}{2}\left\Vert\frac{\mathcal{E}_{k}1^n}{n}\right\Vert^2\\
& \quad \leq \mathbb{E}f\left(\frac{X_{k}\textbf{1}_n}{n}\right)+\frac{1}{2}\mathbb{E}\left\Vert \nabla f\left(\frac{X_{k}\textbf{1}_n}{n}\right) \right\Vert^2+\frac{L+1}{2}\left\Vert\mathcal{E}_{k}\right\Vert_{F}^{2}\left\Vert\frac{1^n}{n}\right\Vert^2\\
& \quad \leq \mathbb{E}f\left(\frac{X_{k}\textbf{1}_n}{n}\right)+\frac{1}{2}\mathbb{E}\left\Vert \nabla f\left(\frac{X_{k}\textbf{1}_n}{n}\right) \right\Vert^2+\frac{L+1}{2}ng(k)\frac{1}{n}\\
& \quad = \mathbb{E}f\left(\frac{X_{k}\textbf{1}_n}{n}\right)+\frac{1}{2}\mathbb{E}\left\Vert \nabla f\left(\frac{X_{k}\textbf{1}_n}{n}\right) \right\Vert^2+\frac{L+1}{2}g(k)\\
\end{split}  
\end{equation}

where we used the inequality $\left\Vert\mathcal{E}_{k}\right\Vert_{F}^{2}\leq ng(k)$. $\left\Vert \; \right\Vert_F$ is the Frobenius norm.

Now we bound $T_1$ as:
\begin{align}
&T_1 
=
{\mathbb{E}\left\Vert\nabla f\left(\frac{\hat{X}_{k}\textbf{1}_n}{n}\right)-\frac{\partial f \left(\hat{X}_{k}\right)\textbf{1}_n}{n}\right\Vert^2} 
\leq
\frac{1}{n^2}\sum_{i=1}^{n}{\mathbb{E}\left\Vert \nabla f_i\left(\frac{\sum_{j=1}^{n}{\hat{x}_{k,j}}}{n}\right)-\nabla f_i\left(\hat{x}_{k,i}\right)\right\Vert^2} \nonumber 
\\ \leq
&\frac{L^2}{n^2}\sum_{i=1}^{n}{\underbrace{\mathbb{E}\left\Vert  \frac{\sum_{j=1}^{n}{\hat{x}_{k,j}}}{n}- \hat{x}_{k,i}\right\Vert^2}_{:=\hat{Q}_{k,i}}},
\end{align}
where $\hat{Q}_{k,i}$ is the square distance of the broadcasted variable $i$ to the average of all broadcasted local variables. From~\eqref{eqn:error_state} and Lemma~\ref{lemma:2a2b_bound}, we can conclude that $\hat{Q}_{k,i} \leq 2Q_{k,i} + 2Q^{\epsilon}_{k,i}$, i.e.,
\begin{align}
T_1 
\leq 
\frac{2L^2}{n^2}
\left(\sum_{i=1}^{n}{\underbrace{\mathbb{E}\left\Vert  \frac{\sum_{j=1}^{n}{x_{k,j}}}{n}- x_{k,i}\right\Vert^2}_{Q_{k,i}}}
+
\sum_{i=1}^{n}{\underbrace{\mathbb{E}\left\Vert  \frac{\sum_{j=1}^{n}{\epsilon_{k,j}}}{n}- \epsilon_{k,i}\right\Vert^2}_{Q^{\epsilon}_{k,i}}}\right)
\end{align}
where $\epsilon_{k,i}$ contains $\epsilon_{k,i,I}$ for all parameters $I = \{1,\dots,N\}$. Now we can bound $Q^{\epsilon}_{k,i}$ as:
\begin{align}
Q^{\epsilon}_{k,i} \leq ng(k)
\end{align}
which implies
\begin{align}
\label{eqn:QhatQng}
\hat{Q}_{k,i} \leq 2Q_{k,i} + 2ng(k).
\end{align}

Now we need to find a bound for $Q_{k,i}$

\begin{align}
&Q_{k,i} = \mathbb{E}\left\Vert  \frac{\sum_{j=1}^{n}{x_{k,j}}}{n}-x_{k,i}\right\Vert^2 \nonumber
\\ &=
\mathbb{E}\left\Vert  \frac{X_k \textbf{1}_n}{n}-X_k e_i\right\Vert^2 \text{where} \quad e_i \quad \text{is the one-hot encoded vector} \nonumber
\\ &=
\mathbb{E}\left\Vert \frac{\hat{X}_{k-1}W \textbf{1}_n - \gamma \partial F\left(\hat{X}_{k-1};\xi_{k-1}\right) \textbf{1}_n}{n} - \hat{X}_{k-1}We_i + \gamma \partial F\left(\hat{X}_{k-1};\xi_{k-1}\right)e_i\right\Vert^2 \nonumber
\\ &=
\mathbb{E}\left\Vert \frac{X_{k-1} \textbf{1}_n+\mathcal{E}_{k-1} \textbf{1}_n-\gamma \partial F\left(\hat{X}_{k-1};\xi_{k-1}\right) \textbf{1}_n}{n} - X_{k-1}We_i - \mathcal{E}_{k-1}We_i + \gamma \partial F\left(\hat{X}_{k-1};\xi_{k-1}\right)e_i\right\Vert^2 \nonumber 
\\ &=
\mathbb{E}\left\Vert \frac{X_{0} \textbf{1}_n + \sum_{k=0}^{K-1}{\mathcal{E}_{k} \textbf{1}_n} - \gamma \sum_{k=0}^{K-1}{\partial F\left(\hat{X}_{k};\xi_{k}\right) \textbf{1}_n}}{n} - X_{0}W^ke_i - \sum_{k=0}^{K-1}{\mathcal{E}_{k}We_i} \right. \nonumber
\\
&\quad \left.+ \gamma \sum_{k=0}^{K-1}{\partial F\left(\hat{X}_{k};\xi_{k}\right)W^{K-1+k}e_i}\right\Vert^2 \nonumber
\\ &\stackrel{X_0=0}{=}
\mathbb{E}\left\Vert \sum_{k=0}^{K-1}{\frac{\mathcal{E}_{k} \textbf{1}_n}{n}} - \sum_{k=0}^{K-1}{\mathcal{E}_{k}We_i} - \gamma \sum_{k=0}^{K-1}{\partial F\left(\hat{X}_{k};\xi_{k}\right)\left(\frac{\textbf{1}_n}{n} - W^{K-1+k}e_i\right)}\right\Vert^2 \nonumber
\\ &\leq
\mathbb{E}\left\Vert \sum_{k=0}^{K-1}{\mathcal{E}_{k}\left(\frac{\textbf{1}^n}{n} - {We_i}\right)}\right\Vert^2 + \gamma^2\mathbb{E}\left\Vert \sum_{k=0}^{K-1}{\partial F\left(\hat{X}_{k};\xi_{k}\right)\left(\frac{\textbf{1}_n}{n} - W^{K-1+k}e_i\right)}\right\Vert^2 \nonumber
\\ &\leq
\mathbb{E} \sum_{k=0}^{K-1}{\left\Vert\mathcal{E}_{k}\right\Vert_{F}^2\left\Vert \left(\frac{\textbf{1}^n}{n}-{We_i}\right)\right\Vert^2} + 2\mathbb{E} \sum_{k\neq k^\prime}^{}{\left\Vert\mathcal{E}_{k}\right\Vert_{F}\left\Vert \left(\frac{\textbf{1}^n}{n}-{We_i}\right)\right\Vert \left\Vert\mathcal{E}_{k^\prime}\right\Vert_{F}\left\Vert \left(\frac{\textbf{1}^n}{n}-{We_i}\right)\right\Vert} \nonumber
\\ & \quad +
\gamma^2\mathbb{E}\left\Vert \sum_{k=0}^{K-1}{\partial F\left(\hat{X}_{k};\xi_{k}\right)\left(\frac{\textbf{1}_n}{n}-W^{K-1+k}e_i\right)}\right\Vert^2 \nonumber
\\ &\leq
\mathbb{E} \sum_{k=0}^{K-1}{ng(k)\rho}+2\mathbb{E} \sum_{k\neq k^\prime}^{}{ \sqrt{ng(k)}\sqrt{\rho}\sqrt{\rho}\sqrt{ng(k^\prime)} }+\gamma^2\mathbb{E}\left\Vert \sum_{k=0}^{K-1}{\partial F\left(\hat{X}_{k};\xi_{k}\right)\left(\frac{\textbf{1}_n}{n}-W^{K-1+k}e_i\right)}\right\Vert^2 \nonumber
\\ &\leq n\rho\sum_{k=0}^{K-1}{g(k)}+n\rho\sum_{k\neq k^\prime}^{}{2\sqrt{g(k)g(k^\prime)}}+\gamma^2\mathbb{E}\left\Vert \sum_{k=0}^{K-1}{\partial F\left(\hat{X}_{k};\xi_{k}\right)\left(\frac{\textbf{1}_n}{n}-W^{K-1+k}e_i\right)}\right\Vert^2 \nonumber
\\ &\leq
n\rho\underbrace{\left(\sum_{k=0}^{K-1}{\sqrt{g(k)}}\right)^2}_{:=\left(G_{1/2}(K-1)\right)^2}+2\gamma^2\underbrace{\mathbb{E}\left\Vert \sum_{k=0}^{K-1}{\left(\partial F\left(\hat{X}_{k};\xi_{k}\right)-\partial f(\hat{X}_k)\right)\left(\frac{\textbf{1}_n}{n}-W^{K-1+k}e_i\right)}\right\Vert^2}_{:=T_2} \nonumber
\\ &\quad +
2\gamma^2\underbrace{\mathbb{E}\left\Vert \sum_{k=0}^{K-1}{\partial f\left(\hat{X}_{k}\right)\left(\frac{\textbf{1}_n}{n}-W^{K-1+k}e_i\right)}\right\Vert^2}_{:=T_3} \nonumber,
\end{align}

where $G_{1/2}(k) = \sum_{k=0}^K \sqrt{g(k)}$ as defined before. We then bound $T_2$ as follows:

\begin{align}
T_2 &= \mathbb{E}\left\Vert \sum_{k=0}^{K-1}{\left(\partial F\left(\hat{X}_{k};\xi_{k}\right)-\partial f(\hat{X}_k)\right)\left(\frac{\textbf{1}_n}{n}-W^{K-1+k}e_i\right)}\right\Vert^2 \nonumber
\\ &=
\sum_{k=0}^{K-1}{\mathbb{E}\left\Vert\left(\partial F\left(\hat{X}_{k};\xi_{k}\right)-\partial f(\hat{X}_k)\right)\left(\frac{\textbf{1}_n}{n}-W^{K-1+k}e_i\right)\right\Vert^2} \nonumber
\\ &\leq \sum_{k=0}^{K-1}{\mathbb{E}\left\Vert\left(\partial F\left(\hat{X}_{k};\xi_{k}\right)-\partial f(\hat{X}_k)\right)\right\Vert_{F}^2 \left\Vert \left(\frac{\textbf{1}_n}{n}-W^{K-1+k}e_i\right)\right\Vert^2} \nonumber
\\ &\leq \sum_{k=0}^{K-1}{n\sigma^2 \left\Vert \left(\frac{\textbf{1}_n}{n}-W^{K-1+k}e_i\right)\right\Vert^2} \leq n\sigma^2\sum_{k=0}^{K-1}{\rho^{K-1+k}} \leq \frac{n\sigma^2}{1-\rho}
\end{align}

The bound for $T_3$ is as follows:

\begin{align}
T_3 &= \mathbb{E}\left\Vert \sum_{k=0}^{K-1}{\partial f(\hat{X}_k)\left(\frac{\textbf{1}_n}{n}-W^{K-1+k}e_i\right)}\right\Vert^2 \nonumber
\\ &=
\underbrace{\sum_{k=0}^{K-1}{\mathbb{E}\left\Vert\partial f(\hat{X}_k)\left(\frac{\textbf{1}_n}{n}-W^{K-1+k}e_i\right)\right\Vert}^2}_{:=T_4} \nonumber
\\ &\quad +
\underbrace{\sum_{k\neq k^\prime}^{K-1}{\mathbb{E}\left<\partial f(\hat{X}_k)\left(\frac{\textbf{1}_n}{n}-W^{K-1+k}e_i\right),\partial f(\hat{X}_{k^\prime})\left(\frac{\textbf{1}_n}{n}-W^{K-1+k^\prime}e_i\right)\right>}}_{:=T_5}
\end{align}

We will bound $T_4$ and $T_5$ separately. $T_4$ is bound as:

\begin{align}
T_4 &= \sum_{k=0}^{K-1}{\mathbb{E}\left\Vert\partial f(\hat{X}_k)\left(\frac{\mathbf{1}_n}{n}-W^{K-1+k}e_i\right)\right\Vert}^2 \nonumber
\\ &\leq
\sum_{k=0}^{K-1}{\mathbb{E}\left\Vert\partial f(\hat{X}_k)\right\Vert^2\left\Vert\left(\frac{\mathbf{1}_n}{n}-W^{K-1+k}e_i\right)\right\Vert}^2 \nonumber
\\ &\stackrel{\text{Lemma~\ref{lemma:break_df}}}{\leq}
3\sum_{k=0}^{K-1}\sum_{i=1}^{n}{\mathbb{E}L^2\hat{Q}_{k,i}\left\Vert \frac{\mathbf{1}_n}{n}-W^{K-1+k}e_i \right\Vert^2}+\frac{3n\varsigma^2}{1-\rho} \nonumber
\\ &\quad +
3\sum_{k=0}^{K-1}{\mathbb{E}\left\Vert \nabla f\left( \frac{\hat{X}_k\mathbf{1}_n}{n} \right)\mathbf{1}^\mathsf{T}_n \right\Vert^2\left\Vert \frac{\mathbf{1}_n}{n}-W^{K-1-k} \right\Vert^2}    
\end{align}

$T_5$ is bound as:

\begin{align}
T_5 &= \sum_{k\neq k^\prime}^{K-1}\mathbb{E}\left<\partial f(\hat{X}_k)\left(\frac{\textbf{1}_n}{n}-W^{K-1+k}e_i\right),\partial f(\hat{X}_{k^\prime})\left(\frac{\textbf{1}_n}{n}-W^{K-1+k^\prime}e_i\right)\right> \nonumber
\\ &\leq
\sum_{k\neq k^\prime}^{K-1}{\mathbb{E}\left\Vert\partial f(\hat{X}_k)\right\Vert\left\Vert\frac{\textbf{1}_n}{n}-W^{K-1+k}e_i\right\Vert\left\Vert\partial f(\hat{X}_{k^\prime})\right\Vert\left\Vert\frac{\textbf{1}_n}{n}-W^{K-1+k^\prime}e_i\right\Vert} \nonumber
\\ &\leq
\sum_{k\neq k^\prime}^{K-1}{\mathbb{E}\left(\frac{\left\Vert\partial f(\hat{X}_k)\right\Vert^2}{2}+\frac{\left\Vert\partial f(\hat{X}_{k^\prime})\right\Vert^2}{2}\right)\rho^{K-1-\frac{k+k^\prime}{2}}} \nonumber
\\ &\leq
\sum_{k\neq k^\prime}^{K-1}{\mathbb{E}\left(\left\Vert\partial f(\hat{X}_k)\right\Vert^2\right)\rho^{K-1-\frac{k+k^\prime}{2}}} \nonumber
\\ &\stackrel{\text{Lemma~\ref{lemma:break_df}}}{\leq}
\underbrace{3\sum_{k\neq k^\prime}^{K-1}{\left( \sum_{i=1}^{n}{\mathbb{E}L^2\hat{Q}_{k,i}}+\left\Vert \nabla f\left( \frac{\hat{X}_k\mathbf{1}_n}{n} \right)\mathbf{1}^\mathsf{T}_n \right\Vert^2 \right)\rho^{K-1-\frac{k+k^\prime}{2}}}}_{:=T_6}+\underbrace{\sum_{k\neq k^\prime}^{K-1}{3n\varsigma^2\rho^{K-1-\frac{k+k^\prime}{2}}}}_{:=T_7}
\end{align}

Now $T_6$ is bounded as follows:

\begin{align}
T_6 &= 3\sum_{k\neq k^\prime}^{K-1}{\left( \sum_{i=1}^{n}{\mathbb{E}L^2\hat{Q}_{k,i}}+\left\Vert \nabla f\left( \frac{\hat{X}_k\mathbf{1}_n}{n} \right)\mathbf{1}^\mathsf{T}_n \right\Vert^2 \right)\rho^{K-1-\frac{k+k^\prime}{2}}} \nonumber
\\ &\stackrel{\eqref{eqn:QhatQng}}{\leq}
6\sum_{k=0}^{K-1}{\left( \sum_{i=1}^{n}{2\mathbb{E}L^2Q_{k,i}}+2n^2L^2g(k)+\left\Vert \nabla f\left( \frac{\hat{X}_k\mathbf{1}_n}{n} \right)\mathbf{1}^\mathsf{T}_n \right\Vert^2 \right) \sum^{K-1}_{k^\prime=k+1}{\sqrt{\rho}^{2K-2-k-k^\prime}}} \nonumber
\\ &\leq
6\sum_{k=0}^{K-1}{\left( \sum_{i=1}^{n}{2\mathbb{E}L^2Q_{k,i}}+2n^2L^2g(k)+\left\Vert \nabla f\left( \frac{\hat{X}_k\mathbf{1}_n}{n} \right)\mathbf{1}^\mathsf{T}_n \right\Vert^2 \right) \frac{\sqrt{\rho}^{K-1-k}}{1-\sqrt{\rho}}}
\end{align}

And $T_7$ is bounded as such:

\begin{align}
T_7 &= 
6n\varsigma^2\sum_{k > k^\prime}^{K-1}{\rho^{K-1-\frac{k+k^\prime}{2}}}
=
6n\varsigma^2\frac{\left( \rho^{\frac{k}{2}}-1 \right)\left( \rho^{\frac{k}{2}}-\sqrt{\rho} \right)}{\left( \sqrt{\rho}-1\right)^2\left(\sqrt{\rho}+1 \right)}
\\ &\leq
6n\varsigma^2 \frac{1}{\left(1-\sqrt{\rho}\right)^2}
\end{align}

Now combining $T_6$ and $T_7$ into $T_5$, and then $T_5$ and $T_4$ into $T_3$, we obtain the following upper bound:

\begin{align}
T_3 &\leq 3\sum_{k=0}^{K-1}\sum_{i=1}^{n}{\mathbb{E}L^2\hat{Q}_{k,i}\left\Vert \frac{\mathbf{1}_n}{n}-W^{K-1+k}e_i \right\Vert^2} \nonumber
\\ &\quad +
\frac{3n\varsigma^2}{1-\rho} + 3\sum_{k=0}^{K-1}{\mathbb{E}\left\Vert \nabla f\left( \frac{\hat{X}_k\mathbf{1}_n}{n} \right)\mathbf{1}^\mathsf{T}_n \right\Vert^2\left\Vert \frac{\mathbf{1}_n}{n}-W^{K-1-k} \right\Vert^2} \nonumber
\\ &\quad +
6\sum_{k=0}^{K-1}{\left( \sum_{i=1}^{n}{2\mathbb{E}L^2Q_{k,i}}+2n^2L^2g(k)+\left\Vert \nabla f\left( \frac{\hat{X}_k\mathbf{1}_n}{n} \right)\mathbf{1}^\mathsf{T}_n \right\Vert^2 \right) \frac{\sqrt{\rho}^{K-1-k}}{1-\sqrt{\rho}}}+\frac{6n\varsigma^2}{\left(1-\sqrt{\rho}\right)^2} \nonumber
\\ &\leq
3\sum_{k=0}^{K-1}\sum_{i=1}^{n}{\mathbb{E}L^2\hat{Q}_{k,i}\left\Vert \frac{\mathbf{1}_n}{n}-W^{K-1+k}e_i \right\Vert^2}+ 3\sum_{k=0}^{K-1}{\mathbb{E}\left\Vert \nabla f\left( \frac{\hat{X}_k\mathbf{1}_n}{n} \right)\mathbf{1}^\mathsf{T}_n \right\Vert^2\left\Vert \frac{\mathbf{1}_n}{n}-W^{K-1-k} \right\Vert^2} \nonumber
\\ &\quad + 6\sum_{k=0}^{K-1}{\left( \sum_{i=1}^{n}{2\mathbb{E}L^2Q_{k,i}}+2n^2L^2g(k)+\left\Vert \nabla f\left( \frac{\hat{X}_k\mathbf{1}_n}{n} \right)\mathbf{1}^\mathsf{T}_n \right\Vert^2 \right) \frac{\sqrt{\rho}^{K-1-k}}{1-\sqrt{\rho}}} +  \frac{9n\varsigma^2}{\left(1-\sqrt{\rho}\right)^2} \nonumber 
\\ &\leq 3\sum_{k=0}^{K-1}\sum_{i=1}^{n}{\mathbb{E}L^2\hat{Q}_{k,i}\left\Vert \frac{\mathbf{1}_n}{n}-W^{K-1+k}e_i \right\Vert^2}+ 3\sum_{k=0}^{K-1}{\mathbb{E}\left\Vert \nabla f\left( \frac{\hat{X}_k\mathbf{1}_n}{n} \right)\mathbf{1}^\mathsf{T}_n \right\Vert^2\left\Vert \frac{\mathbf{1}_n}{n}-W^{K-1-k} \right\Vert^2} \nonumber
\\ &\quad +
6\sum_{k=0}^{K-1}{\left( \sum_{i=1}^{n}{2\mathbb{E}L^2Q_{k,i}}+\left\Vert \nabla f\left( \frac{\hat{X}_k\mathbf{1}_n}{n} \right)\mathbf{1}^\mathsf{T}_n \right\Vert^2 \right) \frac{\sqrt{\rho}^{K-1-k}}{1-\sqrt{\rho}}} \nonumber
\\ &\quad + 
\frac{9n\varsigma^2}{\left(1-\sqrt{\rho}\right)^2}+12n^2L^2G(K-1)
\end{align}

As a reminder, we have defined the following expressions as $G(k) = \sum_{k=0}^K g(k), \;
G_{1/2}(k) = \sum_{k=0}^K \sqrt{g(k)}$. Now we plug the above bound, along with the bound on $T_2$ to obtain the following:

\begin{align}
Q_{k,i} &\leq  n\rho G_{1/2}^2(K-1)
+ 
24\gamma^2n^2L^2 G(K-1) 
+ 
\frac{2\gamma^2n\sigma^2}{1-\rho} \nonumber
\\ &\quad +
6\gamma^2\sum_{k=0}^{K-1}\sum_{i=1}^{n}{\mathbb{E}L^2\hat{Q}_{k,i}\left\Vert \frac{\mathbf{1}_n}{n} - W^{K-1-k}e_i \right\Vert^2} \nonumber
\\ &\quad +
6\gamma^2\sum_{k=0}^{K-1}{\mathbb{E}\left\Vert \nabla f\left( \frac{\hat{X}_k\mathbf{1}_n}{n} \right)\mathbf{1}^\mathsf{T}_n \right\Vert^2\left\Vert \frac{\mathbf{1}_n}{n}-W^{K-1-k} \right\Vert^2} \nonumber
\\ &\quad +
12\gamma^2\sum_{k=0}^{K-1}{\left( \sum_{i=1}^{n}{2\mathbb{E}L^2Q_{k,i}}
+
\mathbb{E}\left\Vert \nabla f\left( \frac{\hat{X}_k\mathbf{1}_n}{n} \right)\mathbf{1}^\mathsf{T}_n \right\Vert^2 \right) \frac{\sqrt{\rho}^{K-1-k}}{1-\sqrt{\rho}}}
+
\frac{18\gamma^2n\varsigma^2}{\left(1-\sqrt{\rho}\right)^2} \nonumber
\\ &\leq
n\rho G_{1/2}^2(K-1) + 24\gamma^2n^2L^2G(K-1)+\frac{2\gamma^2n\sigma^2}{1-\rho}+\frac{18\gamma^2n\varsigma^2}{\left(1-\sqrt{\rho}\right)^2} \nonumber
\\ &\quad + 12\gamma^2\sum_{k=0}^{K-1}\sum_{i=1}^{n}{\mathbb{E}L^2\left(Q_{k,i}+ng(k)\right)\rho^{K-1-k}} +  6\gamma^2\sum_{k=0}^{K-1}{\mathbb{E}\left\Vert \nabla f\left( \frac{\hat{X}_k\mathbf{1}_n}{n} \right)\mathbf{1}^\mathsf{T}_n \right\Vert^2\rho^{K-1-k}} \nonumber
\\ &\quad +
12\gamma^2\sum_{k=0}^{K-1}{\left( \sum_{i=1}^{n}{2\mathbb{E}L^2Q_{k,i}}+\mathbb{E}\left\Vert \nabla f\left( \frac{\hat{X}_k\mathbf{1}_n}{n} \right)\mathbf{1}^\mathsf{T}_n \right\Vert^2 \right) \frac{\sqrt{\rho}^{K-1-k}}{1-\sqrt{\rho}}} \nonumber
\\ &\leq 
n\rho G_{1/2}^2(K-1)+12\gamma^2n^2(2L^2+1)G(K-1)+\frac{2\gamma^2n\sigma^2}{1-\rho}+\frac{18\gamma^2n\varsigma^2}{\left(1-\sqrt{\rho}\right)^2} \nonumber
\\ &\quad + 12\gamma^2\sum_{k=0}^{K-1}{\mathbb{E}\left\Vert \nabla f\left( \frac{\hat{X}_k\mathbf{1}_n}{n} \right)\mathbf{1}^\mathsf{T}_n \right\Vert^2 \left(\rho^{K-1-k}+\frac{2\sqrt{\rho}^{K-1-k}}{1-\sqrt{\rho}}\right)} \nonumber
\\ &\quad +
12\gamma^2\sum_{k=0}^{K-1}{ \sum_{i=1}^{n}{\mathbb{E}L^2Q_{k,i}} \left(\rho^{K-1-k}+\frac{2\sqrt{\rho}^{K-1-k}}{1-\sqrt{\rho}}\right)}.
\end{align}

Let $M_k=\mathbb{E}\sum^{n}_{i=1}{\frac{Q_{k,i}}{n}}$,i.e., the expected average of $Q_k$ in all nodes. We then adjust the previous equation to obtain a bound on $M_k$ as follows:

\begin{align}
M_k &\leq n\rho G_{1/2}^2(K-1)+12\gamma^2n^2(2L^2+1)G(K-1)+\frac{2n\gamma^2\sigma^2}{1-\rho}+\frac{18n\gamma^2\varsigma^2}{\left(1-\sqrt{\rho}\right)^2} \nonumber
\\ &\quad +
12\gamma^2\sum_{k=0}^{K-1}{\mathbb{E}\left\Vert \nabla f\left( \frac{\hat{X}_k\mathbf{1}_n}{n} \right)\mathbf{1}^\mathsf{T}_n \right\Vert^2 \left(\rho^{K-1-k}+\frac{2\sqrt{\rho}^{K-1-k}}{1-\sqrt{\rho}}\right)} \nonumber
\\ &\quad +
12\gamma^2L^2n\sum_{k=0}^{K-1}{ M_k \left(\rho^{K-1-k}+\frac{2\sqrt{\rho}^{K-1-k}}{1-\sqrt{\rho}}\right)}
\end{align}

Now we sum it from $k=0$ to $K-1$ to obtain the following:
\begin{align}
&\sum_{k=0}^{K-1}{M_k} 
\leq 
Kn\rho G_{1/2}^2(K-1)+12K\gamma^2n^2(2L^2+1)G(K-1)+\frac{2Kn\gamma^2\sigma^2}{1-\rho}+\frac{18Kn\gamma^2\varsigma^2}{\left(1-\sqrt{\rho}\right)^2} \nonumber
\\ &\quad +
12\gamma^2\sum_{k=0}^{K-1}{\sum_{i=0}^{K-1}{\mathbb{E}\left\Vert \nabla f\left( \frac{\hat{X}_i\mathbf{1}_n}{n} \right)\mathbf{1}^\mathsf{T}_n \right\Vert^2 \left(\rho^{K-1-i}+\frac{2\sqrt{\rho}^{K-1-i}}{1-\sqrt{\rho}}\right)}} \nonumber
\\ &\quad +
12n\gamma^2L^2\sum_{k=0}^{K-1}{\sum_{i=0}^{K-1}{M_i \left(\rho^{K-1-i}+\frac{2\sqrt{\rho}^{K-1-i}}{1-\sqrt{\rho}}\right)}} \nonumber
\\ &\leq 
Kn\rho G_{1/2}^2(K-1)+12K\gamma^2n^2(2L^2+1)G(K-1)+\frac{2Kn\gamma^2\sigma^2}{1-\rho}+\frac{18Kn\gamma^2\varsigma^2}{\left(1-\sqrt{\rho}\right)^2} \nonumber
\\ &\quad +
12\gamma^2\sum_{k=0}^{K-1}{\mathbb{E}\left\Vert \nabla f\left( \frac{\hat{X}_k\mathbf{1}_n}{n} \right)\mathbf{1}^\mathsf{T}_n \right\Vert^2 \left(\sum_{i=0}^{K-1}{\rho^{K-1-i}}+\frac{2\sum_{i=0}^{K-1}{\sqrt{\rho}^{K-1-i}}}{1-\sqrt{\rho}}\right)} \nonumber
\\ &\quad +
12n\gamma^2L^2\sum_{k=0}^{K-1}{M_k \left(\sum_{i=0}^{K-1}{\rho^{K-1-i}}+\frac{2\sum_{i=0}^{K-1}{\sqrt{\rho}^{K-1-i}}}{1-\sqrt{\rho}}\right)} \nonumber
\\ &\leq 
Kn\rho G_{1/2}^2(K-1)+12K\gamma^2n^2(2L^2+1)G(K-1)+\frac{2Kn\gamma^2\sigma^2}{1-\rho}+\frac{18Kn\gamma^2\varsigma^2}{\left(1-\sqrt{\rho}\right)^2} \nonumber
\\ &\quad +
\frac{36\gamma^2}{\left(1-\sqrt{\rho}\right)^2}\sum_{k=0}^{K-1}{\mathbb{E}\left\Vert \nabla f\left( \frac{\hat{X}_k\mathbf{1}_n}{n} \right)\mathbf{1}^\mathsf{T}_n \right\Vert^2}+\frac{36n\gamma^2L^2}{\left(1-\sqrt{\rho}\right)^2}\sum_{k=0}^{K-1}{M_k}
\end{align}

Rearranging the terms, the expression becomes:

\begin{align}
\left(1-\frac{36n\gamma^2L^2}{\left(1-\sqrt{\rho}\right)^2}\right)\sum_{k=0}^{K-1}{M_k} 
&\leq 
Kn\rho G_{1/2}^2(K-1)+12K\gamma^2n^2(2L^2+1)G(K-1)+\frac{2Kn\gamma^2\sigma^2}{1-\rho} \nonumber
\\ &\quad +
\frac{18Kn\gamma^2\varsigma^2}{\left(1-\sqrt{\rho}\right)^2} 
+
\frac{36\gamma^2}{\left(1-\sqrt{\rho}\right)^2}\sum_{k=0}^{K-1}{\mathbb{E}\left\Vert \nabla f\left( \frac{\hat{X}_k\mathbf{1}_n}{n} \right)\mathbf{1}^\mathsf{T}_n \right\Vert^2}
\end{align}

Defining $C_2 = \left(1-\frac{36 n \gamma^2 L^2}{\left(1-\sqrt{\rho}\right)^2}\right)$, we can rewrite it to:

\begin{align}
\label{eqn:sumMk}
\sum_{k=0}^{K-1}{M_k} 
&\leq
C_{2}^{-1}\left(Kn\rho G_{1/2}^2(K-1)
+
12K\gamma^2n^2(2L^2+1)G(K-1)\right)
+
\frac{2Kn\gamma^2\sigma^2}{C_2(1-\rho)} \nonumber
\\ &\quad +
\frac{18Kn\gamma^2\varsigma^2}{C_2 \left(1-\sqrt{\rho}\right)^2}
+
\frac{36\gamma^2}{C_2\left(1-\sqrt{\rho}\right)^2}\sum_{k=0}^{K-1}{\mathbb{E}\left\Vert \nabla f\left( \frac{\hat{X}_k\mathbf{1}_n}{n} \right)\mathbf{1}^\mathsf{T}_n \right\Vert^2}
\end{align}

We also know that $T_1$ is bounded as such:

\begin{align}
T_1 &\leq
\frac{2L^2}{n^2}\left(\sum_{i=1}^{n}{Q_{k,i}}+\sum_{i=1}^{n}{ng(k)}\right) \nonumber
\\ &\leq
\frac{2L^2}{n}\left(M_k+ng(k)\right)
\end{align}

We then input this bound on $T_1$ in~\eqref{lipschitzbound} and obtain:

\begin{align}
&\quad \mathbb{E}f\left(\frac{X_{k+1}\textbf{1}_n}{n}\right) \nonumber \\
&\quad \leq \mathbb{E}f\left(\frac{X_{k}\textbf{1}_n}{n}\right)+\frac{1}{2}\mathbb{E}\left\Vert \nabla f\left(\frac{X_{k}\textbf{1}_n}{n}\right) \right\Vert^2
+\frac{L+1}{2}g(k)+\frac{\gamma^2L\sigma^2}{2n}
- \frac{\gamma}{2}\mathbb{E}\left\Vert\nabla f\left(\frac{\hat{X}_{k}\textbf{1}_n}{n}\right)\right\Vert^2 \nonumber \\
&\quad + \frac{\gamma^2L-\gamma}{2}\mathbb{E}\left\Vert \frac{\partial f(\hat{X}_k)\textbf{1}_n}{n}\right\Vert^2+\frac{\gamma L^2}{n}M_k+\gamma L^2g(k) \nonumber \\
&\quad \leq \mathbb{E}f\left(\frac{X_{k}\textbf{1}_n}{n}\right)+\frac{\gamma L^2}{2}\mathbb{E}\left\Vert \frac{\xi_{k}\textbf{1}_n}{n} \right\Vert^2+\frac{L+1}{2}g(k)+\frac{\gamma^2L\sigma^2}{2n}- \frac{1-\gamma}{2}\mathbb{E}\left\Vert\nabla f\left(\frac{X_{k}\textbf{1}_n}{n}\right)\right\Vert^2 \nonumber \\
&\quad + \frac{\gamma^2L-\gamma}{2}\mathbb{E}\left\Vert \frac{\partial f(\hat{X}_k)\textbf{1}_n}{n}\right\Vert^2+\frac{\gamma L^2}{n}M_k+\gamma L^2g(k) \nonumber \\
&\quad \leq \mathbb{E}f\left(\frac{X_{k}\textbf{1}_n}{n}\right)+\frac{3\gamma L^2+L+1}{2}g(k)+\frac{\gamma^2L\sigma^2}{2n}- \frac{1-\gamma}{2}\mathbb{E}\left\Vert\nabla f\left(\frac{X_{k}\textbf{1}_n}{n}\right)\right\Vert^2 \nonumber \\
&\quad + \frac{\gamma^2L-\gamma}{2}\mathbb{E}\left\Vert \frac{\partial f(\hat{X}_k)\textbf{1}_n}{n}\right\Vert^2+\frac{\gamma L^2}{n}M_k
\end{align}

Summing from $k=0$ to $k=K-1$ on both sides yields:
\begin{equation}
\begin{split}
&\quad \frac{1-\gamma}{2}\sum_{k=0}^{K-1}{\mathbb{E}\left\Vert\nabla f\left(\frac{X_{k}\textbf{1}_n}{n}\right)\right\Vert^2}- \frac{\gamma^2L-\gamma}{2}\sum_{k=0}^{K-1}{\mathbb{E}\left\Vert \frac{\partial f(\hat{X}_k)\textbf{1}_n}{n}\right\Vert^2} \\
&\quad \leq f(0)-f^*+ \frac{K\gamma^2L\sigma^2}{2n}+\frac{\gamma L^2}{n}\sum_{k=0}^{K-1}{M_k}+\frac{3\gamma L^2+L+1}{2}G(K-1) \\
&\quad \stackrel{\eqref{eqn:sumMk}}{\leq} f(0)-f^*+ \frac{K\gamma^2L\sigma^2}{2n}+C_2^{-1}K\rho\gamma L^2G_{1/2}^2(K-1) \\
& \quad +\left(12C_2^{-1}K\gamma^3nL^2(2L^2+1)+\frac{3\gamma L^2+L+1}{2}\right)G(K-1)
+
\frac{2K\gamma^3n\sigma^2L^2}{C_2(1-\rho)} \\
& \quad +
\frac{18Kn\gamma^3\varsigma^2L^2}{C_2\left(1-\sqrt{\rho}\right)^2}+\frac{36\gamma^3L^2}{C_2\left(1-\sqrt{\rho}\right)^2}\sum_{k=0}^{K-1}{\mathbb{E}\left\Vert \nabla f\left( \frac{\hat{X}_k\mathbf{1}_n}{n} \right) \right\Vert^2}\\
\end{split}
\end{equation}

Rearranging the terms, dividing by $K$ and using the Lipschitz inequality  $\left\Vert \nabla f\left( \frac{\hat{X}_k\mathbf{1}_n}{n} \right) \right\Vert^2 \leq 2\left\Vert \nabla f\left( \frac{X_k\mathbf{1}_n}{n} \right)\right\Vert^2 + 2L^2g(k)$ (similar to what we have done in~\eqref{lipschitzbound}), we obtain:
\begin{equation}
\begin{split}
&\quad \frac{1}{K}\left(\frac{1-\gamma}{2}-\frac{72\gamma^3L^2}{C_2\left(1-\sqrt{\rho}\right)^2}\right)\sum_{k=0}^{K-1}{\mathbb{E}\left\Vert\nabla f\left(\frac{X_{k}\textbf{1}_n}{n}\right)\right\Vert^2}+ \frac{\gamma-\gamma^2 L}{2K}\sum_{k=0}^{K-1}{\mathbb{E}\left\Vert \frac{\partial f(\hat{X}_k)\textbf{1}_n}{n}\right\Vert^2}\\
& \quad \leq 
\frac{f(0)-f^*}{K}+ \frac{\gamma^2 L\sigma^2}{2n} \\
& \quad +
\left(12C_2^{-1}\gamma^3nL^2\left(2L^2+1\right)+\frac{3\gamma L^2+L+1}{2K}+\frac{72\gamma^3L^4}{KC_2\left(1-\sqrt{\rho}\right)^2}\right)G(K-1)\\
& \quad +
C_2^{-1}\gamma\rho L^2G_{1/2}^2(K-1)+\frac{2n\gamma^3\sigma^2L^2}{C_2(1-\rho)}+\frac{18n\gamma^3\varsigma^2L^2}{C_2\left(1-\sqrt{\rho}\right)^2}\\
\end{split}
\end{equation}

where $C_1 = \left(\frac{1-\gamma}{2}-\frac{72\gamma^3}{C_2\left(1-\sqrt{\rho}\right)^2}L^2\right)$. This completes the proof.
\end{proof}

\begin{proof}[\textbf{Proof to Corollary~\ref{corr:conv}}]
First we want to remove the term
$\sum_{k=0}^{K-1}{\mathbb{E}\left\Vert \frac{\partial f(\hat{X}_k)\textbf{1}_n}{n}\right\Vert^2}$ in the LHS and maintain the inequality. For that, the coefficient of that term has to satisfy:
\begin{align}
\label{eqn:gammalessL}
&\frac{\gamma-\gamma^2 L}{2K} > 0 \nonumber \\
\implies &\gamma < \frac{1}{L}
\end{align}
Now we have
\begin{equation}
\begin{split}
&\quad \frac{C_1}{K}\sum_{k=0}^{K-1}{\mathbb{E}\left\Vert\nabla f\left(\frac{X_{k}\textbf{1}_n}{n}\right)\right\Vert^2}\\
& \quad \leq \frac{f(0)-f^*}{K}+ \frac{\gamma^2 L\sigma^2}{2n} \\
& \quad +
\left(12C_2^{-1}\gamma^3nL^2\left(2L^2+1\right)+\frac{3\gamma L^2+L+1}{2K}+\frac{72\gamma^3L^4}{KC_2\left(1-\sqrt{\rho}\right)^2}\right)G(K-1)\\
& \quad +C_2^{-1}\gamma\rho L^2G_{1/2}^2(K-1)+\frac{2n\gamma^3\sigma^2L^2}{C_2(1-\rho)}+\frac{18n\gamma^3\varsigma^2L^2}{C_2\left(1-\sqrt{\rho}\right)^2}\\
\end{split}
\end{equation}

We choose $\gamma = \frac{1}{2\rho L^2\sqrt{K}+\sigma\sqrt{K/n}}$. Also $\gamma$ should satisfy $\gamma < 1$. In order to satisfy that as well as~\eqref{eqn:gammalessL}, we enforce
\begin{align}
&\frac{1}{2\rho L^2\sqrt{K}+\sigma\sqrt{K/n}} < \frac{1}{L+1} \nonumber \\
\implies &K > \left( \frac{\sqrt{n}(L+1)}{2 \rho L^2 \sqrt{n} + \sigma} \right)^2
\end{align}

The $\gamma$ satisfies the following as well:
\begin{align}
&\gamma < \frac{1}{\sigma \sqrt{\frac{K}{n}}} \nonumber \\
\implies &\gamma^2 < \frac{n}{\sigma^2 K}
\end{align}
Since $\gamma < 1$, we also have $\gamma^3\leq \frac{n}{\sigma^2K}$.

Now if $K \geq \frac{72 L^2 n^2}{\sigma^2 (1 - \sqrt{\rho})^2}$, we can bound $C_2$ as $C_2 \geq \frac{1}{2}$. Further we can also bound $C_1$ as follows:
\begin{align}
C_1 = &\left(\frac{1-\gamma}{2}-\frac{72\gamma^3}{C_2\left(1-\sqrt{\rho}\right)^2}L^2\right) \nonumber \\
& \geq \left(\frac{1-\gamma}{2}-\frac{144L^2\gamma^3}{\left(1-\sqrt{\rho}\right)^2}\right) \geq \frac{1}{2}
\end{align}

Finally we have:
\begin{equation}
\begin{split}
&\quad \frac{1}{2K}\sum_{k=0}^{K-1}{\mathbb{E}\left\Vert\nabla f\left(\frac{X_{k}\textbf{1}_n}{n}\right)\right\Vert^2}\leq \frac{\left(f(0)-f^*+L/2\right)}{K}\\
& \quad +\left(\frac{\left(1-\sqrt{\rho}\right)^2\left(2L^2+1\right)}{3\left(2\rho L^2\sqrt{K}+\sigma\sqrt{K/n}\right)}+\frac{7L^2+L+1}{2K}\right)G(K-1)\\
& \quad +\frac{2\rho L^2G_{1/2}^2(K-1)}{2\rho L^2\sqrt{K}+\sigma\sqrt{K/n}}+\frac{4nL^2}{\left(2\rho L^2\sqrt{K}+\sigma\sqrt{K/n}\right)^3}\left(\frac{\sigma^2}{\left(1-\rho\right)}+\frac{9\varsigma^2}{\left(1-\sqrt{\rho}\right)^2}\right)\\
& \quad \leq \frac{\left(f(0)-f^*+L/2\right)}{K}+\frac{C_3}{\sqrt{K}}G(K-1)+\frac{C_4}{K}G(K-1)\\
& \quad +\frac{1}{\sqrt{K}}G_{1/2}^2(K-1)+\frac{4n^3L^2}{\sigma^3K\sqrt{Kn}}\left(\frac{\sigma^2}{\left(1-\rho\right)}+\frac{9\varsigma^2}{\left(1-\sqrt{\rho}\right)^2}\right)\\
\end{split}
\end{equation}
where $C_3=\frac{\left(1-\sqrt{\rho}\right)^2\left(2L^2+1\right)}{6\rho L^2}$ and $C_4=\frac{7L^2+L+1}{2}$.

If $K$ is large enough, in particular, if $K\geq \frac{4n^3L^2}{\sigma^3\left(f(0)-f^*+L/2\right)}\left(\frac{\sigma^2}{\left(1-\rho\right)}+\frac{9\varsigma^2}{\left(1-\sqrt{\rho}\right)^2}\right)$, the last term is bounded by $\frac{\left(f(0)-f^*+L/2\right)}{\sqrt{Kn}}$. Thus the final expression is

\begin{equation}
\begin{split}
\quad \frac{1}{K}\sum_{k=0}^{K-1}{\mathbb{E}\left\Vert\nabla f\left(\frac{X_{k}\textbf{1}_n}{n}\right)\right\Vert^2} 
\leq 
&\left(2(f(0)-f^*)+L\right)\left(\frac{1}{K}+\frac{1}{\sqrt{Kn}}\right) \\
&+
\left(\frac{2C_3}{\sqrt{K}}+\frac{2C_4}{K}\right)G(K-1)
+
\frac{2}{\sqrt{K}}G_{1/2}^2(K-1)
\end{split}
\end{equation}
which completes the proof.
\end{proof}

\bibliographystyle{unsrt}
\def\bibfont{\small}
\bibliography{research}
\end{document}